\documentclass[nohyperref]{article}

\usepackage{microtype}
\usepackage{graphicx}
\usepackage{subfigure}
\usepackage{booktabs} 

\usepackage{hyperref}

\usepackage[accepted]{icml2023}

\usepackage{amsmath}
\usepackage{amssymb}
\usepackage{mathtools}
\usepackage{amsthm}
\usepackage{tabularx}
\usepackage{multirow}

\usepackage[capitalize,noabbrev]{cleveref}

\theoremstyle{plain}
\newtheorem{theorem}{Theorem}
\newtheorem{proposition}{Proposition}
\newtheorem{lemma}{Lemma}
\newtheorem{corollary}{Corollary}
\theoremstyle{definition}
\newtheorem{definition}{Definition}

\theoremstyle{remark}

\usepackage[textsize=tiny]{todonotes}

\usepackage{amsmath,amsfonts,bm}

\def\eqref#1{equation~\ref{#1}}

\def\1{\bm{1}}

\def\vone{{\bm{1}}}

\DeclareMathAlphabet{\mathsfit}{\encodingdefault}{\sfdefault}{m}{sl}
\SetMathAlphabet{\mathsfit}{bold}{\encodingdefault}{\sfdefault}{bx}{n}

\def\gA{{\mathcal{A}}}

\def\gE{{\mathcal{E}}}

\def\gG{{\mathcal{G}}}

\def\gL{{\mathcal{L}}}

\def\gN{{\mathcal{N}}}
\def\gO{{\mathcal{O}}}

\def\sI{{\mathbb{I}}}

\def\sR{{\mathbb{R}}}

\DeclareMathOperator*{\argmax}{arg\,max}
\DeclareMathOperator*{\argmin}{arg\,min}

\newcommand{\widgraph}[2]{\includegraphics[keepaspectratio,width=#1]{#2}}
\newcommand{\attention}{\mathrm{Att}}

\icmltitlerunning{Self-Attention Amortized Distributional Projection Optimization for Sliced Wasserstein Point-Cloud Reconstruction}

\begin{document}

\twocolumn[
\icmltitle{Self-Attention Amortized Distributional Projection Optimization for Sliced Wasserstein Point-Cloud Reconstruction}



\icmlsetsymbol{equal}{*}

\begin{icmlauthorlist}
\icmlauthor{Khai Nguyen}{equal,yyy}
\icmlauthor{Dang Nguyen}{equal,comp}
\icmlauthor{Nhat Ho}{yyy}
\end{icmlauthorlist}

\icmlaffiliation{yyy}{Department of Statistics and Data Sciences, University of Texas at Austin, USA}
\icmlaffiliation{comp}{VinAI Research}

\icmlcorrespondingauthor{Khai Nguyen}{khainb@utexas.edu}

\icmlkeywords{Machine Learning, ICML}

\vskip 0.3in
]



\printAffiliationsAndNotice{\icmlEqualContribution} 

\begin{abstract}
Max sliced Wasserstein (Max-SW) distance has been widely known as a solution for less discriminative projections of sliced Wasserstein (SW) distance. In applications that have various independent pairs of probability measures, amortized projection optimization is utilized to predict the ``max" projecting directions given two input measures instead of using projected gradient ascent multiple times. Despite being efficient, Max-SW and its amortized version cannot guarantee metricity property due to the sub-optimality of the projected gradient ascent and the amortization gap. Therefore, we propose to replace Max-SW with distributional sliced Wasserstein distance with von Mises-Fisher (vMF) projecting distribution (v-DSW). Since v-DSW is a metric with any non-degenerate vMF distribution, its amortized version can guarantee the metricity when performing amortization. Furthermore, current amortized models  are not permutation invariant  and symmetric. To address the issue, we design amortized models based on self-attention architecture. In particular, we adopt efficient self-attention architectures to make the computation linear in the number of supports. With the two improvements, we derive \textit{self-attention amortized distributional projection optimization} and show its appealing performance in point-cloud reconstruction and its downstream applications.
\end{abstract}

\section{Introduction}
\label{sec:introduction}
Wasserstein distance~\cite{villani2008optimal,peyre2019computational} has been widely recognized in the community of machine learning as an effective tool. For example, Wasserstein distance is used to explore clusters inside data~\cite{ho2017multilevel}, to transfer knowledge between different domains~\cite{courty2016optimal, damodaran2018deepjdot}, to learn generative models~\cite{arjovsky2017wasserstein, tolstikhin2018wasserstein}, to extract features from graphs~\cite{vincent2022template}, to compare datasets~\cite{alvarez2020geometric}, and many other applications. Despite being effective, Wasserstein distance is extremely expensive to compute. In particular, the computational complexity and memory complexity of Wasserstein distance in the discrete case is $\mathcal{O}(m^3 \log m)$ and $\mathcal{O}(m^2)$ respectively with $m$ is the number of supports. The computational problem becomes more severe for applications that require computing the Wasserstein distance \textit{multiple times} on different pairs of measures. Some examples can be named: deep generative modeling~\cite{genevay2018learning}, deep domain adaptation~\cite{bhushan2018deepjdot}, comparing datasets~\cite{alvarez2020geometric}, topic modeling~\cite{huynh2020otlda}, point-cloud reconstruction~\cite{achlioptas2018learning}, and so on.

By adding entropic regularization~\cite{cuturi2013sinkhorn}, an $\varepsilon$-approximation of Wasserstein distance can be obtained in $\mathcal{O}(m^2/\varepsilon^2)$. However, this approach cannot reduce the memory complexity of $\mathcal{O}(m^2)$ due to the storage of the cost matrix. A more efficient approach is based on the closed-form solution of Wasserstein distance in one dimension which is known as sliced Wasserstein distance~\cite{bonneel2015sliced}. Sliced Wasserstein (SW) distance is defined as the expectation of the Wasserstein distance between random one-dimensional push-forward measures from the two original measures. Thanks to the closed-form solution, SW can be solved in $\mathcal{O}(m \log_2 m)$ while having a linear memory complexity $\mathcal{O}(m)$. Moreover, SW is also better than Wasserstein distance in high-dimensional statistical inference. Namely, the sample complexity (statistical estimation rate) of SW is $\mathcal{O}(n^{-1/2})$ compared to $\mathcal{O}(n^{-1/d})$ of Wasserstein distance with $d$ is the number dimension and $n$ is the number of data samples. Due to appealing properties, SW is utilized successfully in various applications e.g., generative modeling~\cite{deshpande2018generative,nguyen2022revisiting,nguyen2023hierarchical}, domain adaptation~\cite{lee2019sliced}, Bayesian inference~\cite{nadjahi2020approximate,yi2021sliced}, point-cloud representation learning~\cite{Nguyen2021PointSetDistances,naderializadeh2021pooling}, and so on.

The downside of SW is that it treats all projections the same due to the usage of a uniform distribution over projecting directions. This choice is inappropriate in practice since there exist projecting directions that cannot discriminate two interested measures~\cite{kolouri2018sliced}. As a solution, max sliced Wasserstein distance (Max-SW)~\cite{deshpande2019max} is introduced by searching for the best projecting direction that can maximize the projected Wasserstein distance. Max-SW needs to use a projected sub-gradient ascent algorithm to find the ``max" slice. Therefore, in applications that need to evaluate Max-SW \textit{multiple times} on \textit{different pairs of measures}, the repeated optimization procedure is costly. For example, this paper focuses on point-cloud reconstruction applications where Max-SW needs to be computed between various pairs of empirical measures over a point-cloud and its reconstructed version.

To address the problem, amortized projection optimization is proposed in~\cite{nguyen2022amortized}. As in other amortized optimization~\cite{ruishu2017,amos2022tutorial} (learning to learn), an amortized model is estimated to predict the best projecting direction given the two input empirical measures. The authors in~\cite{nguyen2022amortized} propose three types of amortized models including linear model, generalized linear model, and non-linear model. The linear model assumes that the ``max" projecting direction is a linear combination of supports of two measures. The generalized linear model injects the linearity through a link function on the supports of two measures while the non-linear model uses multilayer perceptions to have more expressiveness.

Despite performing well in practice, the previous work has not explored the full potential of amortized optimization in the sliced Wasserstein setting. There are two issues in the current amortized optimization framework. Firstly, the sub-optimality of amortized optimization leads to losing the metricity of the projected distance from the predicted projecting direction. In particular, the metricity of Max-SW is only obtained at the global optimum. Therefore, using an amortized model with sub-optimal solutions cannot achieve the metricity for all pairs of measures. Losing metricity property could hurt the performance of downstream applications. Secondly, the current amortized models are not permutation invariant to the supports of two input measures and are not symmetric. The permutation-invariant and symmetry properties are vital since the ``max" projecting direction is also not changed when permuting supports of two input empirical measures and exchanging two input empirical measures. By inducing the permutation-invariance and symmetry to the amortized model, it could help to learn a better amortized model and reduce the amortization gap

In this paper, we focus on overcoming the two issues of the current amortized projection optimization framework. For metricity preservation, we propose \textit{amortized distributional projection optimization} framework which predicts the best distribution over projecting directions. In particular, we do amortized optimization for distributional sliced Wasserstein (DSW) distance~\cite{nguyen2021distributional} with von Mises Fisher (vMF) slicing distribution~\cite{jupp1979maximum} instead of Max-SW. Thanks to the smoothness of vMF, the metricity can be preserved even without a zero amortization gap. For the permutation-invariance and symmetry properties, we propose to use the self-attention mechanism~\cite{vaswani2017attention} to design the amortized model. Moreover, we utilize efficient self-attention approaches that have the computational complexity scales linearly in the number of supports including efficient attention~\cite{shen2021efficient} and linear attention~\cite{wang2020linformer}.

\textbf{Contribution.} In summary, our contribution is two-fold:

1. First, we introduce \textit{amortized distributional projection amortization} framework which predicts the best location parameter for von Mises-Fisher (vMF) distribution in distributional sliced Wasserstein (DSW) distance. Due to the smoothness of vMF, the metricity is guaranteed for all pairs of measures. Moreover, we enhance amortized models by inducing inductive biases which are permutation invariance and symmetry. To improve the efficiency, we leverage two linear-complexity attention mechanisms including efficient attention~\cite{shen2021efficient} and linear attention~\cite{wang2020linformer} to parameterize the amortized model. Combining the above two improvements, we obtain \textit{self-attention amortized distributional projection amortization} framework 

2. Second, we adapt the new  framework to the point-clouds reconstruction problem. In particular, we want to learn an autoencoder that can reconstruct (encode and decode) all point-clouds through their latent representations. The main idea is to treat a point-cloud as an empirical measure and use sliced Wasserstein distances as the reconstruction losses. Here, amortized optimization serves as a fast way to yield informative projecting directions for sliced Wasserstein distance to discriminative all pairs of original point-cloud and reconstructed point-cloud. Empirically, we show that the self-attention amortized distributional projection amortization provides better reconstructed point-clouds on the ModelNet40 dataset~\cite{wu20153d} than the amortized projection optimization framework and widely used distances. Moreover, on downstream tasks, the new framework also leads to higher classification accuracy on ModelNet40 and generates ShapeNet chairs with better quality.
 
\textbf{Organization.} The remainder of the paper is organized as
follows. In Section 2, we provide backgrounds for point-cloud reconstruction and popular distances. In Section 3, we define the new amortized distributional projection optimization framework for the point-cloud reconstruction problem. Section 4 benchmarks the proposed method by extensive experiments on point-cloud reconstruction, transfer learning, and point-cloud generation. Finally, proofs of key results and extra materials are in the supplementary.

\textbf{Notation.} For any $d \geq 2$, we denote $\mathcal{U}(\mathbb{S}^{d-1})$ is the uniform measure over the unit hyper-sphere $\mathbb{S}^{d-1}:=\{\theta \in \mathbb{R}^{d}\mid  ||\theta||_2^2 =1\}$.  For $p\geq 1$, $\mathcal{P}_p(\mathbb{R}^d)$ is the set of all probability measures on $\mathbb{R}^d$ that have finite $p$-moments.  For any two sequences $a_{n}$ and $b_{n}$, the notation $a_{n} = \mathcal{O}(b_{n})$ means that $a_{n} \leq C b_{n}$ for all $n \geq 1$, where $C$ is some universal constant. 
We denote $\theta \sharp \mu$ is the push-forward measures of $\mu$ through the function $f:\mathbb{R}^{d} \to \mathbb{R}$ that is $f(x) = \theta^\top x$.

\section{Preliminaries}
\label{sec:Preliminaries}
We first review the point-cloud reconstruction framework in Section~\ref{subsec:pointcloud_reconstruction}. After that, we discuss famous choices of metrics between two point-clouds in Section~\ref{subsec:metrics}. Finally, we present an adapted definition of the amortized projection optimization framework in the point-cloud reconstruction setting in Section~\ref{subsec:amortized_projection}.

\subsection{Point-Cloud Reconstruction}
\label{subsec:pointcloud_reconstruction}

 \begin{figure}[t]
\begin{center}
    
  \begin{tabular}{c}
  \widgraph{0.45\textwidth}{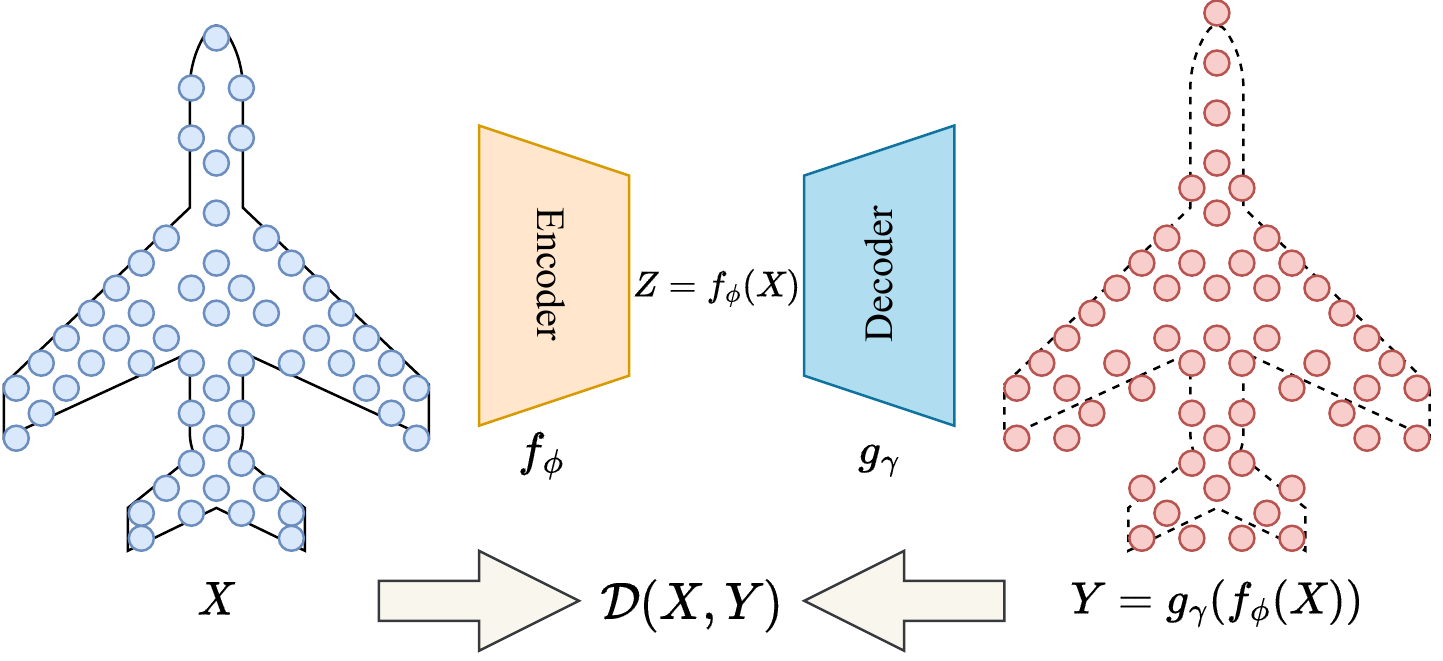}

  \end{tabular}
  \end{center}
  \vskip -0.1in
  \caption{
  \footnotesize{The reconstruction of a point-cloud $X$ (a plane).
}
} 
  \label{fig:reconstruction}
   \vskip -0.2in
\end{figure}

We denote a point-cloud of $m$ points $x_1,\ldots,x_m \in \mathbb{R}^d$ ($d \geq 1$) as $X=(x_1,\ldots,x_m) \in \mathbb{R}^{dm}$ which is a vector of a concatenation of all points in the point-cloud. We denote the set of all possible point-clouds as $\mathcal{X} \subset \mathbb{R}^{dm}$.

\textbf{Permutation invariant metric space.} Given a permutation one-to-one mapping function $\sigma: [m] \to [m]$, we have $\sigma(X) \in \mathcal{X}$ for all $X\in\mathcal{X}$. Moreover, we need a metric $\mathcal{D}:\mathcal{X}\times \mathcal{X} \to \mathbb{R}^+$ such that $\mathcal{D}(X,\sigma(X))=0$ for all $X\in\mathcal{X}$ where $\sigma(X)=(x_{\sigma(1)},\ldots,x_{\sigma(m)})$. Here, $\mathcal{D}$ is a metric, namely, it needs to satisfy the non-negativity, symmetry, triangle inequality, and identity property. The pair $(\mathcal{X},\mathcal{D})$ forms a point-cloud metric space. 

\textbf{Learning representation via reconstruction.} The raw representation of point-clouds is hard to work with in applications due to the complicated metric space. Therefore, a famous approach is to map point-clouds to points in a different space e.g., Euclidean, which is easier to apply machine learning algorithms. In more detail, we want to estimate a function $f_\phi:\mathcal{X} \to \mathcal{Z}$ ($\phi \in \Phi$) where $\mathcal{Z}$ is a set that belongs to another metric space. Then, we can apply machine learning algorithms on $\mathcal{Z}$ instead of $\mathcal{X}$.  The most well-known and effective way to estimate the function $f_\phi$ is through reconstruction loss. Namely, we estimate $f_\phi$ jointly with a function $g_\gamma:\mathcal{Z} \to \mathcal{X} $ ($\gamma \in \Gamma$) given a point-cloud dataset $p(X)$ (distribution over $\mathcal{X}$)  by minimizing the objective:
\begin{align}
    \label{eq:reconstruction}
    \min_{\phi \in \Phi,\gamma \in \Gamma }\mathbb{E}_{X \sim p(X)} \mathcal{D}(X,g_\gamma (f_\phi(X))).
\end{align}
The loss $\mathbb{E}_{X \sim p(X)} \mathcal{D}(X,g_\gamma (f_\phi(X)))$ is known as the reconstruction loss. If the reconstruction loss is 0, we have $g_\gamma = f_\phi^{-1}$ p-almost surely. Therefore, we can move from $\mathcal{X}$ to $\mathcal{Z}$ and move back from $\mathcal{Z}$ to $\mathcal{X}$ without losing information through the functions $f_\phi$ (referred as the encoder) and $g_\gamma$ (referred as the decoder). We show an illustration of the framework~\cite{achlioptas2018learning} in Figure~\ref{fig:reconstruction}. After learning how to do the reconstruction well, other point-cloud tasks can be done using the autoencoder (the pair $(f_\phi,g_\gamma)$) e.g., shape interpolation, shape editing, shape analogy, shape completion, point-cloud classification, and point-cloud generation~\cite{achlioptas2018learning}.

%
\subsection{Metric Spaces for Point-Clouds}
\label{subsec:metrics}
We now review some famous choices of the metric $\mathcal{D}$ which are Chamfer distance~\cite{barrow1977parametric}, Wasserstein distance~\cite{villani2008optimal}, sliced Wasserstein (SW) distance~\cite{bonneel2015sliced}, and max sliced Wasserstein (Max-SW)~\cite{deshpande2019max} distance.

\textbf{Chamfer distance. } For any two point-clouds $X$ and $Y$, the Chamfer distance is defined as follows: $\text{CD}(X, Y)  =$
\begin{align}
\label{eq:chamfer}
     \frac{1}{|X|} \sum \limits_{x \in X} \min \limits_{y \in Y} \| x - y\|_2^{2} 
    + \frac{1}{|Y|} \sum \limits_{y \in Y} \min \limits_{x \in X} \| x - y\|_2^{2},
\end{align}
where $|X|$ denotes the number of points in  $X$.

\textbf{Wasserstein distance.} Given two probability measures $\mu \in \mathcal{P}_p(\mathbb{R}^d)$ and $\nu \in \mathcal{P}_p(\mathbb{R}^d)$, the Wasserstein distance between $\mu$ and $\nu$ is defined as follows: 
\begin{align}
\label{eq:W}
    \text{W}_p(\mu,\nu)  = \left(\inf_{\pi \in \Pi(\mu,\nu)} \int_{\mathbb{R}^d \times \mathbb{R}^d} \| x - y\|_p^{p} d \pi(x,y) \right)^{\frac{1}{p}}
\end{align}
where $\Pi (\mu,\nu)$ is set of all couplings whose marginals are $\mu$ and $\nu$ respectively.
Since the Wasserstein distance is originally defined on probability measures space, we need to convert a point-cloud $X=(x_1,\ldots,x_m) \in \mathcal{X}$ to the corresponding empirical probability measure $P_X =\frac{1}{m}\sum_{i=1}^m \delta_{x_i} \in \mathcal{P}(\mathbb{R}^d)$. Therefore, we can use $\mathcal{D}(X,Y) = \text{W}_p(P_X,P_Y)$ for $X,Y \in \mathcal{X}$.

\textbf{Sliced Wasserstein distance.} As discussed, the Wasserstein distance is expensive to compute with the time complexity $\mathcal{O}(m^3 \log m)$ and the memory complexity $\mathcal{O}(m^2)$. Therefore, an alternative choice is sliced Wasserstein (SW) distance between two probability measures $\mu \in \mathcal{P}_p(\mathbb{R}^d)$ and $\nu\in \mathcal{P}_p(\mathbb{R}^d)$ is:
\begin{align}
\label{eq:SW}
    \text{SW}_p(\mu,\nu)  = \left(\mathbb{E}_{ \theta \sim \mathcal{U}(\mathbb{S}^{d-1})} \text{W}_p^p (\theta \sharp \mu,\theta \sharp \nu)\right)^{\frac{1}{p}
    },
\end{align}
The benefit of SW is that $\text{W}_p (\theta \sharp \mu,\theta \sharp \nu)$ has a closed-form solution which is $$\text{W}_p (\theta \sharp \mu,\theta \sharp \nu)= \left(\int_0^1 |F_{\theta\sharp\mu}^{-1}(z) - F_{\theta \sharp \nu}^{-1}(z)|^{p} dz\right)^{\frac{1}{p}},$$ with $F^{-1}$ denotes the inverse CDF function. The expectation is often approximated by Monte Carlo sampling, namely, it is replaced by the average from $\theta_1,\ldots,\theta_L$ that are drawn i.i.d from $\mathcal{U}(\mathbb{S}^{d-1})$. The computational complexity and memory complexity of SW becomes $\mathcal{O}(Lm\log_2 m)$ and $\mathcal{O}(Lm)$.

\textbf{Max sliced Wasserstein distance. } It is well-known that SW has a lot of less discriminative projections due to the uniform sampling. Therefore, max sliced Wasserstein distance is proposed to use the most discriminative projecting direction. Max sliced Wasserstein (Max-SW) distance~\cite{deshpande2019max} between $\mu \in \mathcal{P}_p(\mathbb{R}^d)$ and $\nu\in \mathcal{P}_p(\mathbb{R}^d)$ is introduced as follows:
\begin{align}
    \label{eq:MaxSW}
    \text{Max-SW}_p(\mu,\nu)=\max_{\theta \in \mathbb{S}^{d - 1}} W_p(\theta\sharp \mu,\theta \sharp \nu),
\end{align}
Max-SW is often computed by a projected sub-gradient ascent algorithm. When the projected sub-gradient ascent algorithm has $T \geq 1$ iterations, the computation complexity of Max-SW is $\mathcal{O}(T m\log_2 m )$ and the memory complexity is $\mathcal{O}(m)$. Both SW and Max-SW are applied successfully in point-cloud reconstruction~\cite{Nguyen2021PointSetDistances}.
\subsection{Amortized Projection Optimization}
\label{subsec:amortized_projection}

\textbf{Amortized Optimization.} We start with the definition of amortized optimization.
\begin{definition}
\label{def:amodel} For each context variable $x$ in the context space $\mathcal{X}$, $\theta^\star(x)$ is the solution of the optimization problem $\theta^\star(x) = \argmin_{\theta \in \Theta} \mathcal{L}(\theta,x)$, where $\Theta$ is the solution space. A parametric function $f_\psi: \mathcal{X} \to \Theta$, where $\psi \in \Psi$, is called an amortized model if 
\begin{align}
\label{eq:famortized}
    f_\psi (x) \approx \theta^\star (x), \quad \forall x \in \mathcal{X}.
\end{align}
The amortized model is trained by the amortized optimization objective which is defined as:
\begin{align}
\label{eq:amortizedobjective}
\min_{\psi \in \Psi} \mathbb{E}_{x \sim p(x)} \mathcal{L}(f_\psi(x),x),
\end{align}
where  $p(x)$ is a probability measure on $\mathcal{X}$ which measures the ``importance" of optimization problems.
\end{definition}

\textbf{Amortized Projection Optimization.} We now revisit the point-cloud reconstruction objective with $\mathcal{D}(X,Y)=\text{Max-SW}_p(P_X,P_Y)$:
\begin{align}
    \label{eq:reconstruction_max}
    \min_{\phi \in \Phi,\gamma \in \Gamma }\mathbb{E} \left[\max_{\theta \in \mathbb{S}^{d-1}}\text{W}_p(\theta \sharp P_X,\theta \sharp P_{g_\gamma (f_\phi(X))})\right],
\end{align}
where the expectation is with respect to $X \sim p(X)$. For each point-cloud $X \in \mathcal{X}$, we need to compute a Max-SW distance with an iterative optimization procedure. Therefore, it is computationally expensive.

Authors in ~\cite{nguyen2022amortized}  propose to use amortized optimization~\cite{ruishu2017,amos2022tutorial} to speed up the problem. Instead of solving all optimization problems independently, an amortized model is trained to predict optimal solutions to all problems. In greater detail, given a parametric function $a_\psi: \mathcal{X}\times \mathcal{X} \to \mathbb{S}^{d-1}$ ($\psi \in \Psi$), the amortized objective is:
\begin{align}
    \label{eq:amortized_reconstruction} 
    \min_{\phi \in \Phi,\gamma \in \Gamma}\max_{ \psi \in \Psi}\mathbb{E}\text{W}_p(\theta_{\psi,\gamma,\phi}\sharp P_X,\theta_{\psi,\gamma,\phi} \sharp P_{g_\gamma (f_\phi(X))}),
\end{align}
where the expectation is with respect to $X \sim p(X)$, and $\theta_{\psi,\gamma,\phi} = a_\psi(X,g_\gamma (f_\phi(X)))$.  The above optimization is solved by an alternative stochastic (projected)-gradient descent-ascent algorithm. Therefore, it is faster to compute in each update iteration of $\phi$ and $\gamma$. It is worth noting that the previous work~\cite{nguyen2022amortized} considers the generative model application which is unstable and hard to understand. Here,  we adapt the framework to the point-cloud reconstruction application which is easier to explore the behavior of amortized optimization. We refer the reader to  Algorithms~\ref{alg:trainingMaxSW}-\ref{alg:trainingamortizedMaxSW} in Appendix~\ref{subsec:training_algorithms} for algorithms on training an autoencoder with Max-SW and amortized projection optimization.

\textbf{Amortized models. } Authors in~\cite{nguyen2022amortized} propose three types of amortized models that are based on the literature on linear models~\cite{christensen2002plane}. In particular, the linear amortized model is defined as:
\begin{definition}\label{def:linear_model} 
Given $X,Y \in  \sR^{dm}$, the \emph{linear amortized model} is defined as:
$$
    a_\psi (X,Y) := \frac{w_0+X'w_1 + Y'w_2}{||w_0+X'w_1 + Y'w_2 ||_2},
$$
where $X'$ and $Y'$ are matrices of size $d\times m$ that are  reshaped from the concatenated vectors $X$ and $Y$ of size $dm$,  $\psi =(w_0,w_1,w_2)$ with  $w_1,w_2 \in \sR^{ m}$, and $w_0 \in \sR^d $ .
\end{definition}
Similarly, the generalized linear amortized model and the non-linear amortized model are defined by injecting non-linearity into the linear model. We review the definitions of the generalized linear amortized model and non-linear amortized model in Definitions~\ref{def:glinear_model}-\ref{def:nonlinear_model} in Appendix~\ref{subsec:additional_amortized_models}.

\textbf{Sub-optimality.} Despite being faster, amortized optimization often cannot recover the global optimum of optimization problems. Namely, we denote $$\theta^\star (X) = \text{argmax}_{\theta \in \mathbb{S}^{d-1}} \text{W}_p(\theta \sharp P_X,\theta  \sharp P_{g_\gamma (f_\phi(X))})$$ and $\psi^\star=$
\begin{align*}
    \argmax_{\psi \in \Psi} \mathbb{E}_{X \in p(X)} \left[\text{W}_p(\theta_{\psi,\gamma,\phi}\sharp P_X,\theta_{\psi,\gamma,\phi} \sharp P_{g_\gamma (f_\phi(X))})\right].
\end{align*} Then, it is well-known that the amortization gap $\mathbb{E}_{X \sim p(X)}[c(\theta^\star (X), a_{\psi^\star}(X,g_\gamma (f_\phi(X))) )] > 0$ for a metric $c: \mathbb{S}^{d-1}\times \mathbb{S}^{d-1} \to \mathbb{R}^+$. A great amortized model is one that can minimize the amortization gap. However, in the amortized projection optimization setting, we cannot obtain $\theta^\star (X)$ since the projected gradient ascent algorithm can only yield the local optimum. Therefore, a careful investigation of the amortization gap is challenging.

\section{Self-Attention Amortized Distributional Projection Optimization}
\label{sec:main_method}

In this section, we propose the self-attention amortized distributional projection optimization framework. First,  we present amortized distributional projection optimization to maintain the metricity property in Section~\ref{subsec:amortized_distributional}. We then introduce self-attention amortized models which are symmetric and permutation invariant in Section~\ref{subsec:selfattention_models}.

\subsection{Amortized Distributional Projection Optimization}
\label{subsec:amortized_distributional}

 \begin{figure}[t]
\begin{center}
    
  \begin{tabular}{c}
  \widgraph{0.45\textwidth}{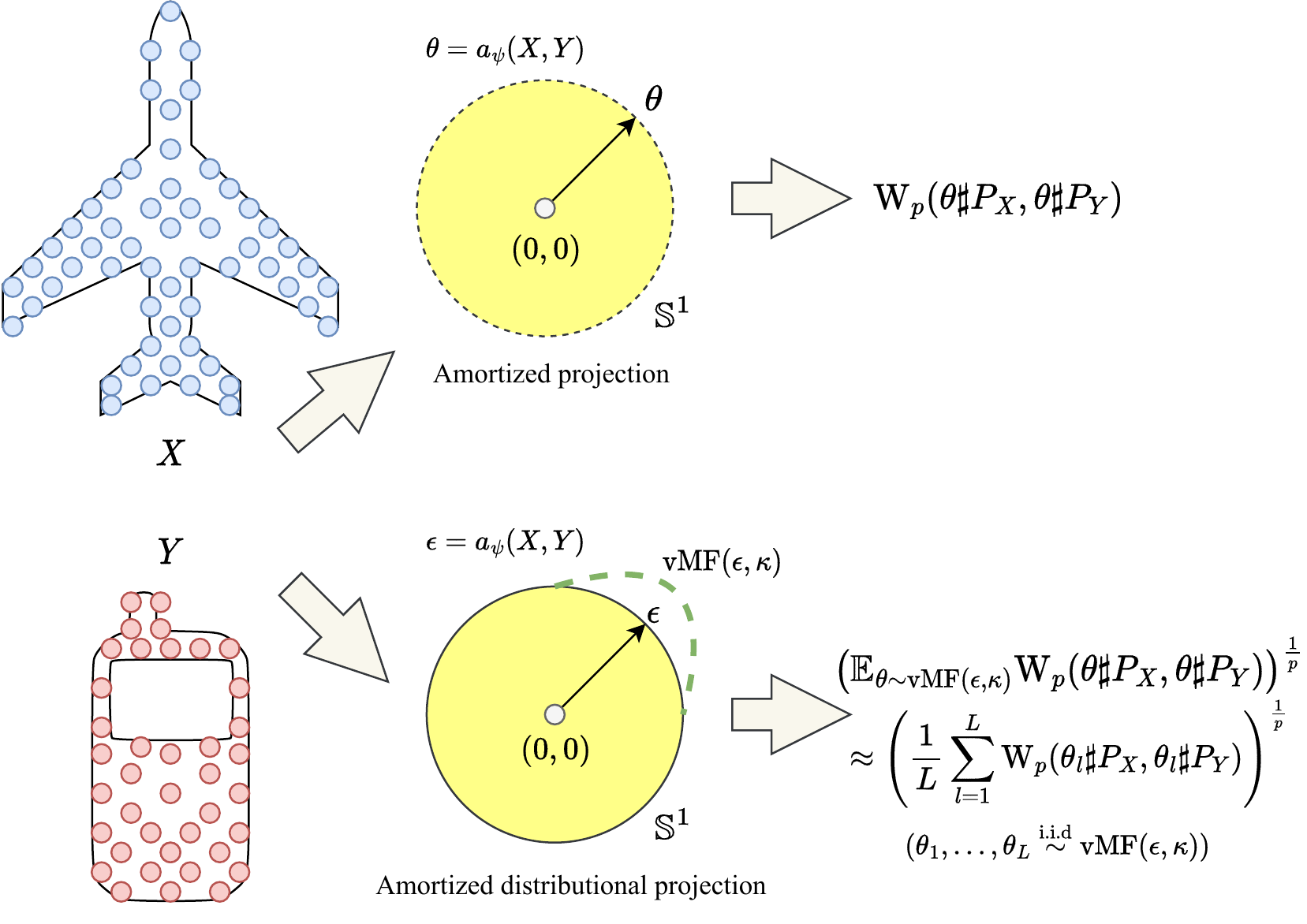}

  \end{tabular}
  \end{center}
  \vskip -0.1in
  \caption{
  \footnotesize{The difference between amortized projection optimization and amortized distributional projection optimization.
}
} 
  \label{fig:amortized}
   \vskip -0.1in
\end{figure}

The current amortized projection optimization framework is for predicting the ``max" projecting direction in Max-SW. However, the projected one-dimensional Wasserstein is only a metric on space of probability measure at the global optimum of Max-SW. Therefore, the local optimum from  the projected sub-gradient ascent algorithm~\cite{nietert2022statistical} and the prediction from the amortized model only yield pseudo-metricity for the projected Wasserstein.

\begin{proposition}
    \label{prop:pseudo_metricity} Let the projected one-dimensional Wasserstein be $\text{PW}_p(\mu,\nu;\hat{\theta}) = \text{W}_p (\hat{\theta} \sharp \mu,\hat{\theta} \sharp \nu)) $ for any  $\mu,\nu \in \mathcal{P}_p(\mathbb{R}^d)$ ($p\geq 1, d\geq 1$) and $\hat{\theta} \in \mathbb{S}^{d-1}$ such that $\hat{\theta} \neq \argmax_{\theta \in \mathbb{S}^{d-1}}\text{W}_p(\theta \sharp \mu,\theta \sharp \nu)$ ,  $\text{PW}_p(\mu,\nu;\hat{\theta})$ is a pseudo metric on $\mathcal{P}_p(\mathbb{R}^d)$ since it satisfies symmetry, non-negativity, triangle inequality, $\mu=\nu$ implies $\text{PW}_p(\mu,\nu;\hat{\theta})=0$, however, $\text{PW}_p(\mu,\nu;\hat{\theta}) = 0  $ does not imply $\mu=\nu$.
\end{proposition}
The proof for Proposition~\ref{prop:pseudo_metricity} is given in Appendix~\ref{subsec:proof:pseudo_metricity}. This result implies that the if reconstruction loss $\mathbb{E}_{X \sim p(X)} [\text{PW}_p(P_X,P_{g_\gamma(f_\phi(X))};\hat{\theta}(X))=0$, it does not imply $X=g_\gamma(f_\phi(X))$ for p-almost surely $X \in \mathcal{X}$. Therefore, a local maximum for $\max_{\theta \in \mathbb{S}^{d-1}}$ in Max-SW reconstruction (Equation~\ref{eq:reconstruction_max}) and the global maximum for $\max_{\psi \in \Psi}$ in amortized Max-SW reconstruction (Equation~\ref{eq:amortized_reconstruction} with a misspecified amortized model) cannot guarantee perfect reconstruction even when their objectives obtain $0$ values.

\textbf{Amortized Distributional Projection Optimization.} To overcome the issue, we propose to replace Max-SW in Equation~\ref{eq:reconstruction_max} with the von Mises Fisher distributional sliced Wasserstein (v-DSW) distance~\cite{nguyen2021improving}:
\begin{align}
\label{eq:vmf_reconstruction}
   & \min_{\phi \in \Phi,\gamma \in \Gamma} \mathbb{E}_{X\sim p(X)} \Big{[} \max_{\epsilon \in \mathbb{S}^{d-1}} \Big{(} \mathbb{E}_{\theta \sim \text{vMF}(\epsilon,\kappa)} 
     \nonumber \\
     & \hspace{4 em} \quad \quad \text{W}_p^p(\theta \sharp P_X,\theta \sharp P_{g_\gamma (f_\phi(X))})\Big{)}^{\frac{1}{p}} \Big{]},
\end{align}
where $\text{vMF}(\epsilon,\kappa)$ is the von Mises Fisher distribution with the mean location parameter $\epsilon \in \mathbb{S}^{d-1}$ and the concentration parameter $\kappa >0$, and $\text{v-DSW}_p(\mu,\nu;\kappa) =\max_{\epsilon \in \mathbb{S}^{d-1}} \Big{(} \mathbb{E}_{\theta \sim \text{vMF}(\epsilon,\kappa)}  \text{W}_p^p(\theta \sharp \mu,\theta \sharp \nu) \Big{)}^{\frac{1}{p}} $ is von Mises Fisher distributional sliced Wasserstein distance. The optimization can be solved by a stochastic projected gradient ascent algorithm with the vMF reparameterization trick. In particular, $\theta_1,\ldots,\theta_L$ ($L\geq 1$) is sampled i.i.d from $\text{vMF}(\epsilon,\kappa)$ via the reparameterized acceptance-rejection sampling~\cite{davidson2018hyperspherical} to approximate $\nabla_{\epsilon} \mathbb{E}_{\text{vMF}(\epsilon,\kappa)}[\text{W}_p^p(\theta \sharp \mu,\theta \sharp \nu)]$ via Monte Carlo integration. We refer the reader to Section~\ref{subsec:vMF} for more detail about the vMF distribution, its sampling algorithm, its reparameterization trick, and the stochastic gradient estimators. We present a visualization of the difference between the new amortized distributional projection optimization framework and the conventional amortized projection optimization framework in Figure~\ref{fig:amortized}. The corresponding amortized objective is:
\begin{align}
\label{eq:amortized_vmf_reconstruction}
   & \min_{\phi \in \Phi,\gamma \in \Gamma}\max_{ \psi \in \Psi}\mathbb{E}_{X \sim p(X)}\Big{(}\mathbb{E}_{\theta \sim \text{vMF}(\epsilon_{\psi,\gamma,\phi},\kappa)} 
     \nonumber \\
     & \hspace{6 em} \quad \quad \text{W}_p^p(\theta \sharp P_X,\theta \sharp P_{g_\gamma (f_\phi(X))})\Big{)}^{\frac{1}{p}},
\end{align}
where $\epsilon_{\psi,\gamma,\phi} = a_\psi(X,g_\gamma (f_\phi(X)))$. The optimization is  solved by an alternative stochastic (projected)-gradient descent-ascent algorithm with the vMF reparameterization.

\begin{theorem}
    \label{theo:reconstruction}
    For any $\epsilon \in \mathbb{S}^{d-1}$ and $0 \leq \kappa < \infty$, if  $\mathbb{E}_{X \sim p(X)} \left(\mathbb{E}_{\theta \sim \text{vMF}(\epsilon,\kappa)}  \text{W}_p^p(\theta \sharp P_X,\theta \sharp P_{g_\gamma (f_\phi(X))}) \right)^{\frac{1}{p}}= 0$, $X=g_\gamma (f_\phi(X))$ for p-almost surely $X \in \mathcal{X}$.
\end{theorem}
The proof of Theorem~\ref{theo:reconstruction} is given in Appendix~\ref{subsec:proof:reconstruction}. The proof is based on proving the metricity of the \textit{non-optimal} von Mises Fisher distributional sliced Wasserstein distance (v-DSW) with the smoothness condition of the vMF distribution. It is worth noting that the proof of metricity of von Mises Fisher distributional sliced Wasserstein distance is new since the original work~\cite{nguyen2021improving} only shows the pseudo-metricity with the global optimality condition. Theorem~\ref{theo:reconstruction} indicates that a perfect reconstruction can be obtained with a local optimum for $\max_{\epsilon \in \mathbb{S}^{d-1}}$ in v-DSW reconstruction (Equation~\ref{eq:vmf_reconstruction}) and a local optimum for $\max_{\psi \in \Psi}$ in amortized v-DSW reconstruction (Equation~\ref{eq:amortized_vmf_reconstruction}).

\textbf{Comparison with SW and Max-SW:} When $\kappa \to 0$, the vMF distribution converges weakly to the uniform distribution over the unit hypersphere. Hence, we can get back the conventional sliced Wasserstein reconstruction in both Equation~\ref{eq:vmf_reconstruction} and Equation~\ref{eq:amortized_vmf_reconstruction}. When $\kappa \to \infty$, vMF distribution converges weakly to the Dirac delta at the location parameter. Therefore, we obtain Max-SW reconstruction and amortized Max-SW reconstruction in Equation~\ref{eq:vmf_reconstruction} and Equation~\ref{eq:amortized_vmf_reconstruction}, respectively. However, when $0 <\kappa<\infty$, v-DSW reconstruction and amortized v-DSW reconstruction can find a region of discriminative projecting directions while preserving the metricity for perfect reconstruction.

\subsection{Self-Attention Amortized Models}
\label{subsec:selfattention_models}

 \begin{figure}[t]
\begin{center}
    
  \begin{tabular}{c}
  \widgraph{0.45\textwidth}{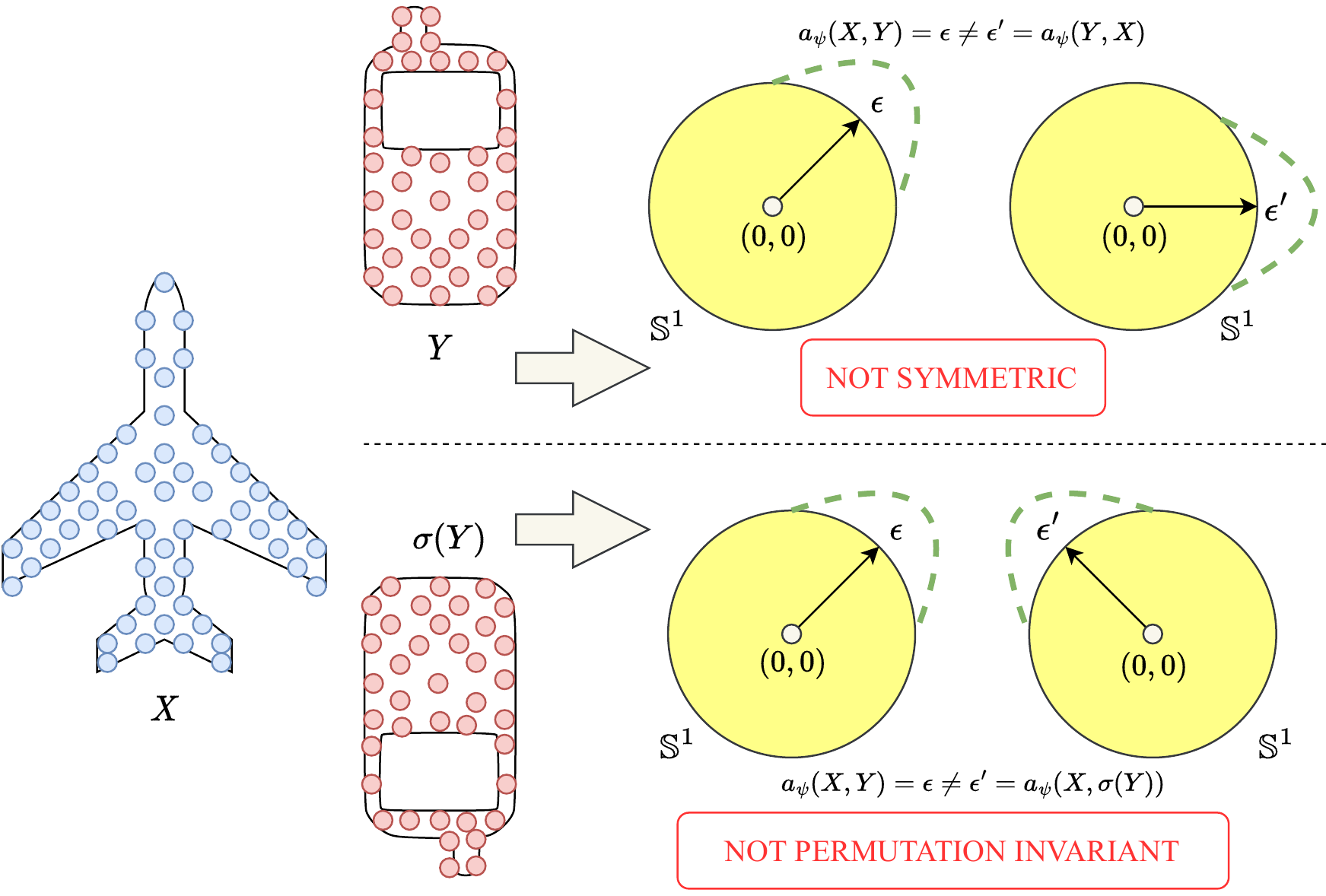}

  \end{tabular}
  \end{center}
  \vskip -0.1in
  \caption{
  \footnotesize{Visualization of an amortized model that is not symmetric and permutation invariant in two dimensions.
}
} 
  \label{fig:invariance}
   \vskip -0.2in
\end{figure}

We now discuss the parameterization of the amortized model for amortized optimization.

\textbf{Permutation Invariance and Symmetry. } Let $X$ and $Y$ be two point-clouds, the optimal slicing distribution $\text{vMF}(\epsilon^\star,\kappa)$ of v-DSW between $P_X$ and $P_Y$ can be obtained by running Algorithm~\ref{alg:trainingvDSW} in Appendix~\ref{subsec:training_algorithms}. Clearly, $\text{vMF}(\epsilon^\star,\kappa)$ is invariant to the permutation of the supports since $P_{\sigma(X)} = P_X$ and $P_{\sigma(Y)} = P_Y$ for a permutation function $\sigma$. Moreover, the optimal slicing distribution $\text{vMF}(\epsilon^\star,\kappa)$ is also unchanged when we exchange $P_X$ and $P_Y$ since v-DSW is symmetric. However, the current amortized models (see Definition~\ref{def:linear_model}, Definitions~\ref{def:glinear_model}-\ref{def:nonlinear_model} in Appendix~\ref{subsec:additional_amortized_models}) are not permutation invariant and symmetric, namely, $a_\psi (X,Y) \neq a_\psi (X,\sigma(Y)) $ and $ a_\psi (X,Y)\neq a_\psi (Y,X)$ . Therefore, the current amortized models could be strongly misspecified. We show a visualization of an amortized model that is not symmetric and permutation invariant in Figure~\ref{fig:invariance}. To address the issue, we propose amortized models that are symmetric and permutation invariant based on the self-attention mechanism.

\textbf{Self-Attention Mechanism.} Attention is well-known for its effectiveness in learning long-range dependencies when data are sequences such as text~\cite{devlin-etal-2019-bert,liu2019roberta,brown2020language} or speech~\cite{li2019improving,wang2020transformer}. This mechanism was then successfully generalized to other data types including image~\cite{carion2020end,dosovitskiy2020image}, video~\cite{sun2019videobert}, graph~\cite{dwivedi2021generalization}, point-cloud~\cite{zhao2021point,guo2021pct}, to name a few.
We now revisit the attention mechanism~\cite{vaswani2017attention}.
Given $Q, K \in \sR^{m \times d_k}, V \in \sR^{m \times d_v}$, the \emph{scaled dot-product attention} operator is defined as:
\begin{align}
\label{eq:scaled_dot_product_attention}
    \attention(Q, K, V) = \text{softmax}_{\rm{row}} \left[ \frac{QK^T}{\sqrt{d_k}} \right] V  
\end{align}
where $\text{softmax}_{\rm{row}}$ denotes the row-wise softmax function. In the self-attention mechanism, the query matrix Q, the key matrix K, and the value matrix V are usually computed by projecting the input sequence X into different subspaces. Thus, the self-attention mechanism is given as follows.
Given $X \in \sR^{m \times d}$, the \emph{self-attention} operator is:
\begin{align}
\label{eq:self_attention}
    \gA_{\zeta}(X) = \attention(XW_q, X W_k, X W_v)
\end{align}
where $W_q, W_k \in \sR^{d \times d_k}, W_v \in \sR^{d \times d_v}$ and $\zeta = (W_q, W_k, W_v)$. The self-attention operator is infamous for its quadratic memory and computational costs. In particular, given an input sequence of length $m$, both the time and space complexity are $\gO(m^2)$. Since we focus on the sliced Wasserstein setting where the computational complexity should be at most $\mathcal{O}(m\log m)$, the conventional self-attention is not appropriate. Several works~\cite{li2020linear,katharopoulos2020transformers,wang2020linformer,shen2021efficient} have been proposed to reduce the overall complexity from $\gO(m^2)$ to $\gO(m)$. In this paper, we utilize two linear complexity variants of attention which are efficient attention~\cite{shen2021efficient} and linear attention~\cite{wang2020linformer}. 
Given $X \in \sR^{m \times d}$, the \emph{efficient self-attention} is defined as: 
\begin{align}
\label{eq:efficient_self_attention}
     &\gE\gA_{\zeta}(X) =\nonumber \\ & \text{softmax}_{\rm{row}}(X W_q) \left[ \text{softmax}_{\rm{col}}(X W_k)^T (X W_v) \right]
\end{align}
where $W_q, W_k \in \sR^{d \times d_k}, W_v \in \sR^{d \times d_v}$, $\zeta = (W_q, W_k, W_v)$, and $\text{softmax}_{\rm{col}}$ denotes applying the softmax function column-wise. 
The \emph{linear self-attention} is:
\begin{align}
\label{eq:linear_self_attention}
    \gL\gA_{\zeta}(X) = \attention(X W_q, W_{k1} X W_{k2}, W_{v1} X W_{v2})
\end{align}
where $W_q, W_{k2} \in \sR^{d \times d_k}, W_{v2} \in \sR^{d \times d_v}$, $W_{k1}, W_{v1} \in \sR^{k \times n}$, and $\zeta = (W_q, W_{k1}, W_{k2}, W_{v1}, W_{v2})$. The projected dimension $k$ is chosen such that $m \gg k$ to reduce the memory and space consumption significantly.


\textbf{Self-Attention Amortized Models:} Based on the self-attention mechanism, we introduce the self-attention amortized model which is permutation invariant and symmetric. Formally, the \textit{self-attention amortized model} is defined as:

\begin{definition}\label{def:efficient_attention_model} 
Given $X,Y \in  \sR^{dm}$, the \emph{self-attention amortized model} is defined as: 
\begin{align}
    a_\psi (X,Y)=\frac{\gA_{\zeta}(X'^\top)^\top \vone_{m} +  \gA_{\zeta}(Y'^\top)^\top \vone_{m}}{||\gA_{\zeta}(X'^\top)^\top  \vone_{m} +  \gA_{\zeta}(Y'^\top)^\top \vone_{m}||_2},
\end{align}
where $X'$ and $Y'$ are matrices of size $d\times m$ that are reshaped from the concatenated vectors $X$ and $Y$ of size $dm$, $\vone_{m}$ is the $m$-dimensional vector whose all entries are $1$ and $\psi =(\zeta)$.
\end{definition}

By replacing the conventional self-attention with the linear self-attention and the efficient self-attention, we obtain the \textit{linear self-attention amortized model} and the \textit{efficient self-attention amortized model}.

\begin{proposition}
    \label{prop:invariance}
    Self-attention amortized models are symmetric and permutation invariant.
\end{proposition}
The proof of Proposition~\ref{prop:invariance} is given in Appendix~\ref{subsec:proof:invariance}. The symmetry follows directly from the definition of the self-attention amortized models. The permutation invariance is proved by showing that the self-attention operators combined with average pooling
are permutation invariant. 

\textbf{Comparison with Set Transformer.} The authors in~\cite{lee2019set} also proposed a method to guarantee the permutation invariant of sets. There are two main differences between our works and theirs. Firstly, Set Transformer introduced a new attention mechanism and a new Transformer architecture while we only present an approach to apply \emph{any} attention mechanism to preserve the permutation invariance property of amortized models. Secondly, Set Transformer maintains the permutation invariance property by using a learnable multi-head attention as the aggregation scheme. We instead still rely on average pooling, a conventional permutation invariant aggregation scheme, to accumulate features learned by self-attention operations. Nevertheless, our works are orthogonal to Set Transformer, in other words, it is possible to apply techniques in Set Transformer to our attention-based amortized models. We leave this investigation for future work. 

\begin{table}[t!]
    \caption{Reconstruction and transfer learning performance on the ModelNet40 dataset. CD and SW are multiplied by 100.}
    \vskip 0.1in
    \label{table:short_reconstruction_result}
    \centering
    \scalebox{0.6}{
        \begin{tabular}{lccccc}
        \toprule
        Method & CD$(10^{-2}, \downarrow)$ & SW$(10^{-2}, \downarrow)$ & EMD$(\downarrow)$ & Acc$ (\uparrow)$ & Time $(\downarrow)$ \\
        \midrule
        CD & 1.25 $\pm$ 0.03 & 681.20 $\pm$ 16.73 & 653.52 $\pm$ 10.43 & 86.28 $\pm$ 0.34 & 95 \\
        EMD & 0.40 $\pm$ 0.00 & 94.54 $\pm$ 2.90 & 168.60 $\pm$ 1.57 & 88.45 $\pm$ 0.20 & 208 \\
        \midrule
        SW & 0.68 $\pm$ 0.01 & 89.61 $\pm$ 3.88 & 191.12 $\pm$ 2.88 & 87.90 $\pm$ 0.27 & 106 \\
        Max-SW & 0.68 $\pm$ 0.01 & 88.22 $\pm$ 1.45 & 190.23 $\pm$ 0.1 & 87.97 $\pm$ 0.14 & 116 \\
        ASW & 0.69 $\pm$ 0.01 & 89.42 $\pm$ 5.07 & 192.03 $\pm$ 3.09 & 87.78 $\pm$ 0.20 & 103 \\
        v-DSW & \textbf{0.67 $\pm$ 0.00} & 85.03 $\pm$ 3.31 & 187.75 $\pm$ 2.00 & 87.83 $\pm$ 0.40 & 633 \\
        $\gL$-Max-SW & 1.06 $\pm$ 0.03 & 121.85 $\pm$ 5.77 & 236.87 $\pm$ 3.42 & 87.70 $\pm$ 0.23 & \textbf{94} \\
        $\gG$-Max-SW & 12.11 $\pm$ 0.29 & 851.07 $\pm$ 2.11 & 829.28 $\pm$ 5.53 & 87.49 $\pm$ 0.36 & 97 \\
        $\gN$-Max-SW & 7.38 $\pm$ 3.29 & 618.74 $\pm$ 153.87 & 648.32 $\pm$ 117.03 & 87.43 $\pm$ 0.15 & 96 \\
        \midrule
        $\gL$v-DSW & 0.68 $\pm$ 0.00 & 85.32 $\pm$ 0.54 & 188.32 $\pm$ 0.23 & 87.70 $\pm$ 0.34 & 114 \\
        $\gG$v-DSW & 0.68 $\pm$ 0.01 & 82.77 $\pm$ 0.48 & 187.04 $\pm$ 1.11 & 87.75 $\pm$ 0.19 & 117 \\
        $\gN$v-DSW & \textbf{0.67 $\pm$ 0.00} & 83.47 $\pm$ 0.49 & 186.66 $\pm$ 0.81 & 87.84 $\pm$ 0.07 & 115 \\
        $\gA$v-DSW & 0.67 $\pm$ 0.01 & 83.08 $\pm$ 1.22 & 186.27 $\pm$ 0.56 & 88.05 $\pm$ 0.17 & 230 \\
        $\gE\gA$v-DSW & 0.68 $\pm$ 0.01 & 82.05 $\pm$ 0.40 & 186.46 $\pm$ 0.25 & 88.07 $\pm$ 0.21 & 125 \\
        $\gL\gA$v-DSW & 0.68 $\pm$ 0.00 & \textbf{81.03 $\pm$ 0.18} & \textbf{185.26 $\pm$ 0.31} & \textbf{88.28 $\pm$ 0.13} & 123 \\
        \bottomrule
        \end{tabular}
    }
    \vskip -0.1in
\end{table}
\begin{figure*}[!t]
\begin{center}
\begin{tabular}{c}
\widgraph{1\textwidth}{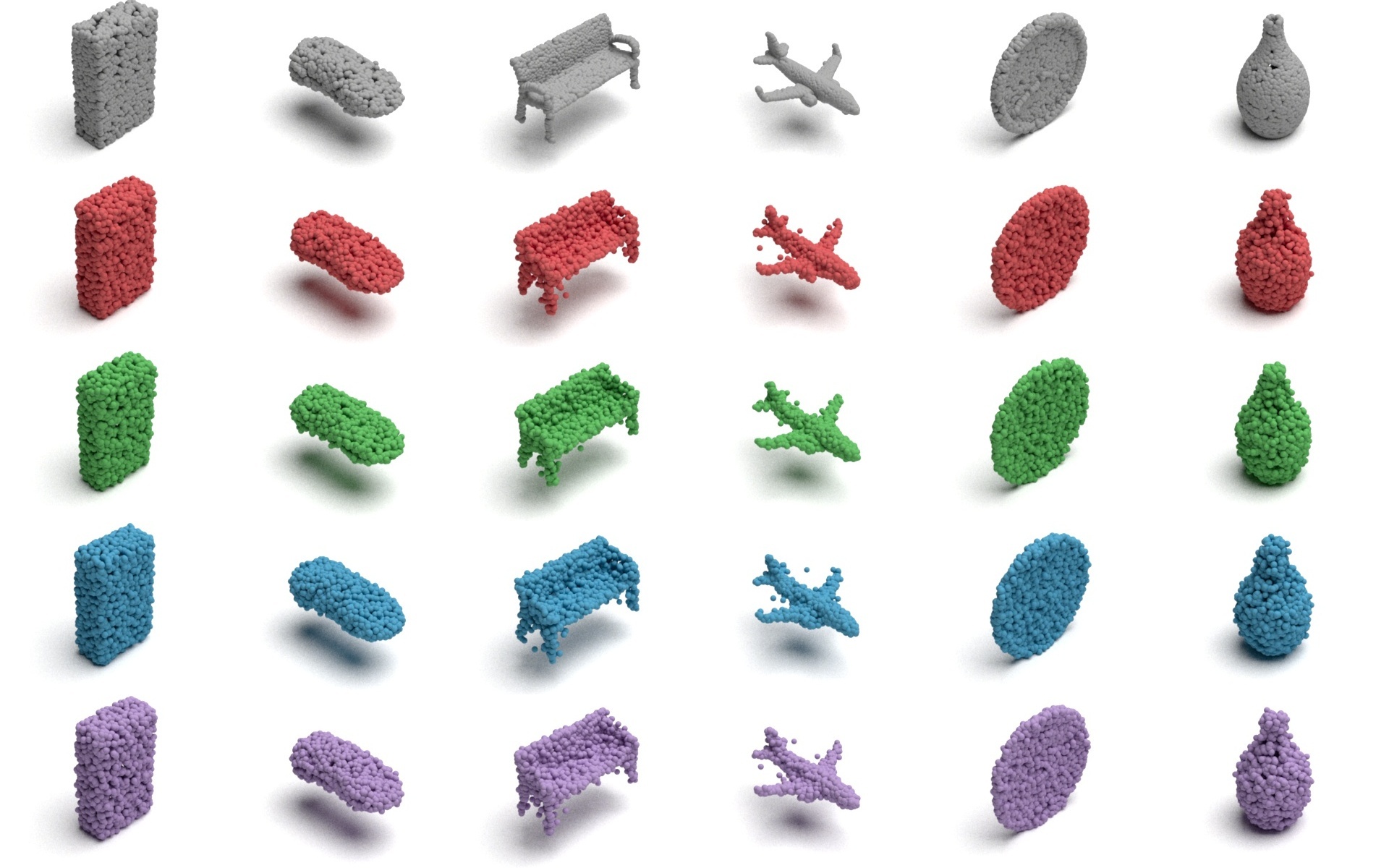}
\end{tabular}
\end{center}
\vskip -0.1in
\caption{
\footnotesize{Qualitative results of reconstructing point-clouds in the ShapeNet Core-55 dataset. From top to bottom, the point-clouds are input, SW, Max-SW (T = 50), v-DSW (T = 50), and $\gL\gA$v-DSW respectively.}
} 
\label{fig:reconstructed_point_clouds_short}
\vskip -0.1in
\end{figure*}
\section{Experiments}
\label{experiments}
To verify the effectiveness of our proposal, we evaluate our methods on the point-cloud reconstruction task and its two downstream tasks including transfer learning and point-cloud generation. Three important questions we want to answer are: 
\begin{enumerate}
    \item  Does the sub-optimality issue of amortized Max-SW occur when working with point-clouds and does replacing Max-SW with v-DSW alleviate the problem? 
    \item Does the proposed amortized distribution projection optimization framework improve the performance over the conventional amortized projection optimization  framework and commonly used distances e.g., Chamfer distance, Earth Mover Distance (Wasserstein distance), SW, Max-SW, adaptive SW (ASW)~\cite{Nguyen2021PointSetDistances}, and v-DSW? \item Are self-attention amortized models better than the previous misspecified amortized models in~\cite{nguyen2022amortized}?
\end{enumerate}
\begin{table}[t!]
    \caption{Comparison between amortized models when approximating von Mises Fisher distributional sliced Wasserstein (v-DSW). T denotes the number of projected sub-gradient ascent iterations.}
    \vskip 0.1in
    \label{table:vDSW_approximation}
    \centering
    \scalebox{0.9}{
        \begin{tabular}{lccc}
        \toprule
        Method & T & Distance $(\uparrow)$ & Time $(\downarrow)$\\
        \midrule
        $\gL$v-DSW & 1 & 52.73 & 0.06 \\
        $\gG$v-DSW & 1 & 50.73 & 0.07 \\
        $\gN$v-DSW & 1 & 51.89 & 0.07 \\
        $\gA$v-DSW & 1 & 53.07 & 1.00 \\
        $\gE\gA$v-DSW & 1 & 53.17 & 0.17 \\
        $\gL\gA$v-DSW & 1 & \textbf{53.83} & 0.14 \\
        \midrule
        v-DSW & 1 & 51.87 & 0.1 \\
        v-DSW & 5 & 51.90 & 0.33 \\
        v-DSW & 10 & 52.65 & 0.5 \\
        v-DSW & 50 & 53.16 & 2.00 \\
        v-DSW & 100 & 54.39 & 4.00 \\
        \bottomrule
        \end{tabular}
    }
    \vskip -0.1in
\end{table}

\begin{table*}[!t]
    \caption{Performance comparison of point-cloud generation on the chair category of ShapeNet. The values of JSD, MMD-CD, and MMD-EMD are multiplied by 100.}
    \vskip 0.1in
    \label{table:short_generation_result}
    \centering
    \scalebox{0.9}{
        \begin{tabular}{lccccccc}
        \toprule
        Method & JSD $(\times 10^{-2}, \downarrow)$ & \multicolumn{2}{c}{MMD $(\times 10^{-2}, \downarrow)$} & \multicolumn{2}{c}{COV $(\%, \uparrow)$} & \multicolumn{2}{c}{1-NNA $(\%, \downarrow)$} \\
        \cmidrule(lr){3-4} \cmidrule(lr){5-6} \cmidrule(lr){7-8}
        & & CD & EMD & CD & EMD & CD & EMD \\
        \midrule
        CD & 17.88 $\pm$ 1.14 & 1.12 $\pm$ 0.02 & 17.19 $\pm$ 0.36 & 23.73 $\pm$ 1.69 & 10.83 $\pm$ 0.89 & 98.45 $\pm$ 0.10 & 100.00 $\pm$ 0.00 \\
        EMD & 5.15 $\pm$ 1.52 & 0.61 $\pm$ 0.09 & 10.37 $\pm$ 0.61 & 41.65 $\pm$ 2.19 & 42.54 $\pm$ 2.42 & 87.76 $\pm$ 1.46 & 87.30 $\pm$ 1.22 \\
        \midrule
        SW & 1.56 $\pm$ 0.06 & 0.72 $\pm$ 0.02 & 10.80 $\pm$ 0.11 & 38.55 $\pm$ 0.43 & 45.35 $\pm$ 0.48 & 89.91 $\pm$ 1.17 & 88.28 $\pm$ 0.70 \\
        Max-SW & 1.63 $\pm$ 0.32 & 0.74 $\pm$ 0.01 & 10.84 $\pm$ 0.08 & 40.47 $\pm$ 1.04 & 47.81 $\pm$ 0.78 & 91.46 $\pm$ 0.72 & 89.93 $\pm$ 0.86 \\
        ASW & 1.75 $\pm$ 0.38 & 0.78 $\pm$ 0.05 & 11.27 $\pm$ 0.38 & 38.16 $\pm$ 2.15 & 45.45 $\pm$ 1.40 & 91.21 $\pm$ 0.40 & 89.36 $\pm$ 0.40 \\
        v-DSW & 1.79 $\pm$ 0.17 & 0.72 $\pm$ 0.02 & 10.73 $\pm$ 0.20 & 37.76 $\pm$ 0.71 & 45.49 $\pm$ 1.37 & 90.23 $\pm$ 0.13 & 88.33 $\pm$ 0.95 \\
        \midrule
        $\gL$v-DSW & 1.67 $\pm$ 0.07 & 0.77 $\pm$ 0.04 & 11.10 $\pm$ 0.33 & 37.91 $\pm$ 1.84 & 45.64 $\pm$ 2.30 & 90.42 $\pm$ 0.53 & 88.82 $\pm$ 0.38 \\
        $\gG$v-DSW & 1.56 $\pm$ 0.22 & 0.75 $\pm$ 0.02 & 10.99 $\pm$ 0.11 & 37.81 $\pm$ 1.70 & 45.69 $\pm$ 0.46 & 90.32 $\pm$ 0.38 & 88.26 $\pm$ 0.28 \\
        $\gN$v-DSW & \textbf{1.44 $\pm$ 0.06} & 0.75 $\pm$ 0.02 & 10.95 $\pm$ 0.09 & 38.40 $\pm$ 1.34 & 46.28 $\pm$ 2.06 & 90.15 $\pm$ 0.80 & 88.65 $\pm$ 0.82 \\
        $\gE\gA$v-DSW & 1.73 $\pm$ 0.21 & \textbf{0.71 $\pm$ 0.04} & \textbf{10.70 $\pm$ 0.26} & 40.03 $\pm$ 1.28 & \textbf{48.01 $\pm$ 1.07} & 89.98 $\pm$ 0.57 & 88.55 $\pm$ 0.38 \\
        $\gL\gA$v-DSW & 1.54 $\pm$ 0.09 & 0.72 $\pm$ 0.03 & 10.74 $\pm$ 0.35 & \textbf{40.62 $\pm$ 1.39} & 45.84 $\pm$ 1.23 & \textbf{89.44 $\pm$ 0.28} & \textbf{87.79 $\pm$ 0.37} \\
        \bottomrule
        \end{tabular}
    }
    \vskip -0.1in
\end{table*}
\textbf{Experiment settings:} Our settings\footnote{Code for the paper  is published at \url{https://github.com/hsgser/Self-Amortized-DSW}.}, which can be found in Appendix~\ref{subsec:ae_details}, are identical to the setting in the paper of ASW. We compare our methods, amortized v-DSW, with the following loss functions: Chamfer discrepancy (CD), Earth-mover distance (EMD), SW, Max-SW, adaptive SW (ASW), v-DSW, and amortized Max-SW variants. For amortized models, we consider 6 different ones. The prefix $\gL, \gG$, and $\gN$ denote the linear, generalized linear, and non-linear amortized models in~\cite{nguyen2022amortized}, respectively. $\gA$, $\gE\gA$, and $\gL\gA$ represent self-attention, efficient self-attention, and linear self-attention, respectively. Implementation details for baseline distances and amortized models are given in Appendices~\ref{subsec:baseline_details} and~\ref{subsec:amortize_details}, respectively. Each experiment was run over three different random seeds. We report the average performance along with the standard deviation for each entity. All experiments are run on NVIDIA V100 GPUs.

\textbf{Comparison with CD and EMD}: The main focus of the paper is to compare the new amortized framework with the conventional amortized framework and sliced Wasserstein variants. The results with CD and EMD are provided only for completeness.  In addition, we found that there is an unfair comparison between EMD and sliced Wasserstein variants in the ASW's paper. In particular, the EMD loss is normalized by the number of points in a point cloud while SW variants are not. To fix the aforementioned issue, we modified the implementation of the EMD loss by scaling it by the number of points (2048 in this case). As a ``perfect" objective, EMD performs better than all SW variants. However, EMD suffers from huge computational costs compared to SW variants

\textbf{Point-cloud reconstruction:} Following ASW~\cite{Nguyen2021PointSetDistances}, we measure the reconstruction performance of different autoencoders on the ModelNet40 dataset~\cite{wu20153d} using three discrepancies: Chamfer discrepancy (CD), sliced Wasserstein distance (SW), and EMD. The quantitative results are summarized in Table~\ref{table:short_reconstruction_result}. For each method, we only report the best performing (based on EMD) model among all choices of hyper-parameters. Full quantitative results (including std) can be found in Table~\ref{table:full_reconstruction_result}. Our methods achieve the best performance in all three discrepancies. In contrast, autoencoders with amortized Max-SW losses fail in this scenario due to the sub-optimality and losing metricity issues that we discussed in Section~\ref{subsec:amortized_projection}. In addition, amortized v-DSW losses have smaller standard deviations over 3 runs than v-DSW. Moreover, using amortized optimization reduces the training time compared to the conventional computation using the projected sub-gradient ascent algorithm (e.g. Max-SW and v-DSW). For example, training one iteration of autoencoder using $\gL\gA$v-DSW only takes 123 seconds while using v-DSW costs 633 seconds. In terms of amortized models, attention-based amortized models lead to lower EMD between reconstruction and input. Qualitative results are given in Figure~\ref{fig:reconstructed_point_clouds_short}, showing the success of our methods in reconstructing 3D point-clouds. Full qualitative results are reported in Figure~\ref{fig:reconstructed_point_clouds_full}.

\textbf{Amortization Gaps:} To validate the advantage of self-attention amortized models over the previous misspecified amortized models, we compare their effectiveness in approximating v-DSW. We create a dataset by sampling 1000 pairs of point-clouds from the ShapeNet Core-55 dataset. Due to the memory constraint when solving amortized optimization, the dataset is divided into 10 batches of size 100. We compute v-DSW and its amortized versions between all pairs of point-clouds and report their average loss values in Table~\ref{table:vDSW_approximation}. Compared to previous misspecified amortized models, attention-based amortized models produce higher losses which are closer to the conventional computation of v-DSW (T = 100). To achieve the same level as efficient/linear self-attention amortized models, one needs to run more than 50 sub-gradient iterations, which is more than 10 times slower. 

\textbf{Transfer learning:} We further feed the latent vectors learned by the above autoencoders into a classifier. Following the settings in ASW's paper, we train our classifier for 500 epochs with a batch size of 256. The optimizer is the same as that in the reconstruction experiment. Table~\ref{table:short_reconstruction_result} illustrates the classification result. Again, we see a boost in accuracy when using self-attention amortized v-DSW.

\textbf{Point-cloud generation:} We also evaluate our methods on the 3D point-cloud generation task. Following~\cite{achlioptas2018learning}, the chair category of ShapeNet is divided into train/valid/test sets in an 85/5/10 ratio. We train each autoencoder on the train set for 100 epochs and evaluate on the valid set. The generator is then trained to generate latent codes learned by the autoencoder, same as~\cite{achlioptas2018learning}. For evaluation, the same set of metrics in~\cite{yang2019pointflow} is used. 
The quantitative results of the test set are given in Table~\ref{table:short_generation_result}. Our methods yield the best performance in all metrics. In addition, attention-based amortized models lead to higher performance than previous amortized models in all metrics except for JSD. Full quantitative results are reported in Table~\ref{table:full_generation_result}.

\section{Conclusion}
\label{conclusion}

We have proposed a self-attention amortized distributional projection optimization framework which uses a self-attention amortized model to predict the best discriminative distribution over projecting direction for each pair of probability measures. The efficient self-attention mechanism helps to inject the geometric inductive biases which are permutation invariance and symmetry into the amortized model while remaining fast computation. Furthermore, the amortized distribution projection optimization framework guarantees the metricity for all pairs of probability measures while the amortization gap still exists. On the experimental side, we compare the new proposed framework to the conventional amortized projection optimization framework and other widely-used distances in the point-cloud reconstruction application and its two downstream tasks including transfer learning and point-cloud generation to show the superior performance of the proposed framework.

\section*{Acknowledgements}
Nhat Ho acknowledges support from the NSF IFML 2019844 and the NSF AI Institute for Foundations of Machine Learning.
\clearpage
\bibliography{example_paper}
\bibliographystyle{icml2023}

\newpage
\appendix
\onecolumn
\begin{center}
{\bf{\Large{Supplement to ``Self-Attention Amortized Distributional Projection Optimization for  Sliced Wasserstein Point-Clouds Reconstruction"}}}
\end{center}
In this supplementary, we first provide some additional materials in Appendix~\ref{sec:additional_materials} including  definitions of generalized linear amortized models and non-linear amortized models in Appendix~\ref{subsec:additional_amortized_models}, the detail of computing von Mises-Fisher distributional sliced Wasserstein in Appendix~\ref{subsec:vMF}, and training algorithms for autoencoders in Appendix~\ref{subsec:training_algorithms}. Next, we collect skipped proofs in the main text in Appendix~\ref{sec:proofs}. After that, we discuss the experimental settings of our experiments in Appendix~\ref{sec:exp_settings}. Finally, we present additional experimental results in Appendix~\ref{sec:add_exps}.

\section{Additional Materials}
\label{sec:additional_materials}



\subsection{Amortized models}
\label{subsec:additional_amortized_models}

We now review the generalized linear amortized model and the non-linear amortized model~\cite{nguyen2022amortized}.
\begin{definition}\label{def:glinear_model} 
Given $X,Y \in  \sR^{dm}$,  the \emph{generalized linear amortized model} is defined as:
\begin{align}
    f_\psi (X,Y) := \frac{g_{\psi_1}(X)' w_1 + g_{\psi_1}(Y)'w_2}{||g_{\psi_1}(X)' w_1 + g_{\psi_1}(Y)' w_2||_2^2},
\end{align}
where $X'$ and $Y'$ are matrices of size $d\times m$ that are  reshaped from the concatenated vectors $X$ and $Y$ of size $dm$, $w_1,w_2 \in \sR^{m}$, $w_0 \in \sR^d $,  $\psi_1 \in \Psi_1$, $g_{\psi_1}: \sR^{dm} \to \sR^{dm}$, $\psi =(w_0,w_1,w_2,\psi_1)$, and  $g_{\psi_1}(X)=(x'_1,\ldots,x'_m)$ and $g_{\psi_1}(Y)=(y'_1,\ldots,y'_m)$. To specify, we let $g_{\psi_1}(X) = (W_2\sigma(W_1 x_1)+b_0 ,\ldots, W_2\sigma(W_1 x_m) +b_0) $, where $\sigma(\cdot)$ is the Sigmoid function, $W_1 \in \mathbb{R}^{d\times d}$, $W_2 \in \mathbb{R}^{d\times d}$, and $b_0 \in \mathbb{R}^d$. 
\end{definition}

\begin{definition}\label{def:nonlinear_model} 
Given $X,Y \in  \sR^{dm}$, the \emph{non-linear amortized model} is defined as:
\begin{align}
    f_\psi (X,Y) := \frac{ h_{\psi_2}(X'w_1 + Y'w_2)}{||h_{\psi_2}(X'w_1 + Y'w_2)||_2^2},
\end{align}
where $X'$ and $Y'$ are matrices of size $d\times m$ that are  reshaped from the concatenated vectors $X$ and $Y$ of size $dm$, $w_1,w_2 \in \sR^{m}$, $\psi_2 \in \Psi_2$, $h_{\psi_2}:\sR^d \to \sR^d$, $\psi =(w_1,w_2,\psi_2)$, and $h_{\psi_2}(x) = W_4 \sigma(W_3x)) +b_0 $ where $\sigma(\cdot)$ is the Sigmoid function.
\end{definition}

\subsection{Von Mises-Fisher distributional sliced Wasserstein distance}
\label{subsec:vMF}
We first start with the definition of von Mises Fisher (vMF) distribution. The von Mises–Fisher distribution (\emph{vMF})\cite{jupp1979maximum} is a probability distribution on the unit hypersphere $\mathbb{S}^{d-1}$  with the density function is :
\begin{align}
\label{def:vMF}
    f(x| \epsilon, \kappa ) : = C_d (\kappa) \exp(\kappa \epsilon^\top x),
\end{align}
where $\epsilon\in \mathbb{S}^{d-1}$ is  the location vector, $\kappa \geq 0$ is the concentration parameter,  and $C_d(\kappa) : = \frac{\kappa^{d/2 -1}}{(2 \pi)^{d/2} I_{d/2 -1 }(\kappa) }$ is the normalization constant. Here, $I_v$ is the modified Bessel function of the first kind at order $v$ \cite{temme2011special}. 

 The vMF distribution is a continuous distribution, its mass concentrates around the mean $\epsilon$, and its density decrease when $x$ goes away from $\epsilon$.  When $\kappa \to 0$, vMF converges in distribution to $\mathcal{U}(\mathbb{S}^{d-1})$, and when $\kappa \to \infty$, vMF converges in distribution to the Dirac distribution centered at $\epsilon$~\cite{Suvrit_directional}.

 \textbf{Reparameterized Rejection Sampling: } The sampling process of vMF distribution is based on the rejection sampling procedure. We review the sampling process in Algorithm~\ref{Alg:vMF_sampling}~\cite{davidson2018hyperspherical,nguyen2021improving}. The algorithm performs the reparameterization for the proposal distribution. We now derive the gradient estimator for $\nabla_\epsilon \mathbb{E}_{\text{vMF}(\theta|\epsilon,\kappa)}\big[f(\theta)\big]$ for a general function $f(\theta)$ to find the maxima $\epsilon^*$ in the optimization problem $\max_{\epsilon \in \mathbb{S}^{d-1}} \mathbb{E}_{\text{vMF}(\theta|\epsilon,\kappa)}\big[f(\theta)\big]$.

 In $d>0$  dimension, let $(\epsilon, \kappa)$ be the parameters of vMF distribution. We denotes  $b=\frac{-2 \kappa+\sqrt{4 \kappa^{2}+(d-1)^{2}}}{d-1}$, two conditional distributions:
 $
    g(\omega \mid \kappa)=\frac{2\left(\pi^{d / 2}\right)}{\Gamma(d / 2)} \mathcal{C}_{d}(\kappa) \frac{\exp (\omega \kappa)\left(1-\omega^{2}\right)^{\frac{1}{2}(d-3)}}{\text{Beta}\left(\frac{1}{2}, \frac{1}{2}(d-1)\right)}, \quad 
    r(\omega|\kappa)= \frac{2 b^{1 / 2(d-1)}}{\text{Beta}\left(\frac{1}{2}(d-1), \frac{1}{2}(d-1)\right)} \frac{\left(1-\omega^{2}\right)^{1 / 2(d-3)}}{[(1+b)-(1-b) \omega]^{d-1}}, \nonumber
$
distribution $s(\psi) := \operatorname{Beta}\left(\frac{1}{2}(d-1), \frac{1}{2}(d-1)\right)$, function $h(\psi, \kappa)=\frac{1-(1+b) \psi}{1-(1-b) \psi}$, distributions  $\pi_1(\psi|\kappa)=s(\psi)\frac{g(h(\psi,\kappa)|\kappa)}{r(h(\psi,\kappa)|\kappa)}$, $\pi_2(v):= \mathcal{U}(\mathbb{S}^{d-2})$, and function  
$
    T(\omega,v,\epsilon)= \Big(I - 2\frac{e_1-\epsilon}{||e_1-\epsilon||_2}\frac{e_1-\epsilon}{||e_1-\epsilon||_2}^\top\Big) \big(\omega, \sqrt{1-\omega^2} v^\top\big)^\top:=\theta. \nonumber
$

We can obtain the gradient estimator by the following  Lemma 2 in \cite{davidsonhyperspherical}:
\begin{align*}
    \nabla_\epsilon \mathbb{E}_{\text{vMF}(\theta|\epsilon,\kappa)}\big[f(\theta)\big]= \nabla_\epsilon \mathbb{E}_{(\psi,v ) \sim \pi_1(\psi|\kappa) \pi_2(v)}\Big[f\big(T(h(\psi,\kappa),v,\epsilon)\big)\Big] =  \mathbb{E}_{(\psi,v ) \sim \pi_1(\psi|\kappa) \pi_2(v)}\Big[\nabla_\epsilon f\big(T(h(\psi,\kappa),v,\epsilon)\big)\Big].
\end{align*}

In v-DSW case, we have $f(\theta) = \text{W}_p^p (\theta \sharp \mu,\theta \sharp \nu)$. Therefore, we have:
 \begin{align*}
     \nabla_\epsilon \mathbb{E}_{\text{vMF}(\theta|\epsilon,\kappa)}\big[\text{W}_p^p (\theta \sharp \mu,\theta \sharp \nu)\big]= \mathbb{E}_{(\psi,v ) \sim \pi_1(\psi|\kappa) \pi_2(v)}\Big[ \nabla_\epsilon \text{W}_p^p (f\big(T(h(\psi,\kappa),v,\epsilon) \sharp \mu,f\big(T(h(\psi,\kappa),v,\epsilon) \sharp \nu) \big)\Big].
 \end{align*}

Then we can get a gradient estimator by using Monte-Carlo estimation scheme:
\begin{align*}
    \nabla_\epsilon \mathbb{E}_{\text{vMF}(\theta|\epsilon,\kappa)}\big[\text{W}_p^p (\theta \sharp \mu,\theta \sharp \nu)\big] \approx \frac{1}{L}\sum_{i=1}^L \Big[ \nabla_\epsilon \text{W}_p^p (f\big(T(h(\psi_i,\kappa),v_i,\epsilon) \sharp \mu,f\big(T(h(\psi_i,\kappa),v_i,\epsilon) \sharp \nu) \big)\Big],
\end{align*}
where $\{\psi_i\}_{i=1}^L \sim \pi_{1}(\psi|\kappa)$ i.i.d, $\{v_i\}_{i=1}^L \sim \pi_2(v)$ i.i.d, and $L$ is the number of projections.
Sampling from $\pi_1(\psi|\kappa)$ is equivalent to the acceptance-rejection scheme in vMF sampling procedure, sampling $\pi_2(v)$  is directly  from $\mathcal{U}(\mathbb{S}^{d-2})$. It is worth noting that the gradient estimator for $\nabla_\kappa \mathbb{E}_{\text{vMF}(\theta|\epsilon,\kappa)}\big[f(\theta)\big]$ can be derived by using the log-derivative trick, however, we do not need it here since we do not optimize for $\kappa$ in v-DSW.

 \begin{algorithm}[t]
  \caption{Sampling from vMF distribution}
  \label{Alg:vMF_sampling}
\begin{algorithmic}
  \STATE {\bfseries Input:} location $\epsilon$, concentration $\kappa$, dimension $d$, unit vector $e_1= (1,0,..,0)$
  \STATE Draw $v \sim \mathcal{U}(\mathbb{S}^{d-2})$ 
\STATE $b \leftarrow \frac{-2 \kappa+\sqrt{4 \kappa^{2}+(d-1)^{2}}}{d-1}$, $a \leftarrow \frac{(d-1)+2 \kappa+\sqrt{4 \kappa^{2}+(d-1)^{2}}}{4}$, $m \leftarrow \frac{4 a b}{(1+b)}-(d-1) \log (d-1)$
  \REPEAT 
    \STATE Draw $\psi \sim \operatorname{Beta}\left(\frac{1}{2}(d-1), \frac{1}{2}(d-1)\right)$
    \STATE $\omega \leftarrow h(\psi, \kappa)=\frac{1-(1+b) \psi}{1-(1-b) \psi}$
    \STATE $t \leftarrow \frac{2 a b}{1-(1-b) \psi}$
    \STATE Draw $u \sim \mathcal{U}([0,1])$
  \UNTIL{{$(d-1) \log (t)-t+m \geq \log (u)$}}
 \STATE $h_1\leftarrow (\omega, \sqrt{1-\omega^2} v^\top)^\top$
\STATE $\epsilon^\prime \leftarrow e_1 - \epsilon$
\STATE $u = \frac{\epsilon^\prime}{||\epsilon^\prime||_2}$
\STATE $U = I - 2uu^\top$
  \STATE {\bfseries Output:} $Uh_1$
\end{algorithmic}
\end{algorithm}

\subsection{Training algorithms}
\label{subsec:training_algorithms}

\textbf{Training point-cloud autoencoder with Max-SW:} We present the algorithm of training autoencoder with Max-SW in Algorithm~\ref{alg:trainingMaxSW}. The algorithm contains a nested loop: one is for training the autoencoder, one is for finding the max projecting direction for Max-SW.

\begin{algorithm}[!t]
\caption{Training point-cloud autoencoder  with  max sliced Wasserstein distance}
\begin{algorithmic}
\label{alg:trainingMaxSW}
\STATE \textbf{Input:} Point-cloud distribution $p(X)$,  learning rate $\eta$, slice learning rate $\eta_s$, model maximum number of iterations $\mathcal{T}$,  slice maximum number of iterations $T$, mini-batch size $k$.
  \STATE \textbf{Initialization:} Initialize the encoder $f_\phi$ and the decoder $g_\gamma$
  \WHILE{$\phi,\gamma$ not converge or reach $\mathcal{T}$}
  \STATE Sample a mini-batch $X_1,\ldots,X_k$ i.i.d from $p(X)$
  \STATE $\nabla_\phi = 0,\nabla_\gamma =0$
  \FOR{$i=1$ to $k$}
  \STATE Initialize $\theta$
   \WHILE{$\theta$ not converge or reach $T$}
  \STATE $\theta = \theta +  \eta_s \cdot \nabla_\theta\text{W}_p (\theta \sharp P_{X_i},\theta \sharp P_{g_\gamma(f_\phi(X_i))})$ \# Other update rules can be used
  \STATE $\theta = \frac{\theta}{||\theta||_2}$ \#Project back to the unit-hypersphere $\mathbb{S}^{d-1}$
  \ENDWHILE
  \STATE $\nabla_\phi = \nabla_\phi+\frac{1}{k}  \nabla_\phi \text{W}_p (\theta \sharp P_{X_i},\theta \sharp P_{g_\gamma(f_\phi(X_i))})$
  \STATE $\nabla_\gamma = \nabla_\gamma+\frac{1}{k}  \nabla_\gamma \text{W}_p (\theta \sharp P_{X_i},\theta \sharp P_{g_\gamma(f_\phi(X_i))})$
  \ENDFOR
  \STATE $\phi = \phi - \eta \cdot \nabla_\phi$  \# Other update rules can be used
  \STATE $\gamma = \gamma - \eta \cdot \nabla_\gamma$  \# Other update rules can be used
  \ENDWHILE
 \STATE \textbf{Return:} $\phi,\gamma$
\end{algorithmic}
\end{algorithm}

\textbf{Training point-cloud autoencoder with amortized projection optimization:} We present the training algorithm for  point-cloud autoencoder with amortized projection optimization in Algorithm~\ref{alg:trainingamortizedMaxSW}. With amortized optimization, the inner loop for finding the max projecting direction is removed.

\begin{algorithm}[!t]
\caption{Training point-cloud autoencoder  with amortized projection optimization}
\begin{algorithmic}
\label{alg:trainingamortizedMaxSW}
\STATE \textbf{Input:} Point-cloud distribution $p(X)$,  learning rate $\eta$, slice learning rate $\eta_s$, model maximum number of iterations $\mathcal{T}$,  mini-batch size $k$.
  \STATE \textbf{Initialization:} Initialize the encoder $f_\phi$, the decoder $g_\gamma$, and the amortized model $a_\psi$
  \WHILE{$\phi,\gamma,\psi$ not converge or reach $\mathcal{T}$}
  \STATE Sample a mini-batch $X_1,\ldots,X_k$ i.i.d from $p(X)$
  \STATE $\nabla_\phi = 0,\nabla_\gamma =0,\nabla_\psi = 0$
  \FOR{$i=1$ to $k$}
  \STATE $\theta_{\psi,\gamma,\phi} = a_\psi (X_i,g_\gamma(f_\phi(X_i)))$
   \STATE $\nabla_\psi = \nabla_\psi+\frac{1}{k}  \nabla_\psi \text{W}_p (\theta_{\psi,\gamma,\phi} \sharp P_{X_i},\theta_{\psi,\gamma,\phi} \sharp P_{g_\gamma(f_\phi(X_i))})$
  \STATE $\nabla_\phi = \nabla_\phi+\frac{1}{k}  \nabla_\phi \text{W}_p (\theta_{\psi,\gamma,\phi} \sharp P_{X_i},\theta_{\psi,\gamma,\phi} \sharp P_{g_\gamma(f_\phi(X_i))})$
  \STATE $\nabla_\gamma = \nabla_\gamma+\frac{1}{k}  \nabla_\gamma \text{W}_p (\theta_{\psi,\gamma,\phi} \sharp P_{X_i},\theta_{\psi,\gamma,\phi} \sharp P_{g_\gamma(f_\phi(X_i))})$
  \ENDFOR
  \STATE $\psi = \psi + \eta_s \cdot \nabla_\psi$  \# Other update rules can be used
  \STATE $\phi = \phi - \eta \cdot \nabla_\phi$  \# Other update rules can be used
  \STATE $\gamma = \gamma - \eta \cdot \nabla_\gamma$  \# Other update rules can be used
  \ENDWHILE
 \STATE \textbf{Return:} $\phi,\gamma$
\end{algorithmic}
\end{algorithm}

\textbf{Training point-cloud autoencoder with v-DSW:} We present the algorithm of training autoencoder with v-DSW in Algorithm~\ref{alg:trainingvDSW}. The algorithm contains a nested loop: one is for training the autoencoder, one is for finding the best distribution over projecting directions for v-DSW.

\begin{algorithm}[!t]
\caption{Training point-cloud autoencoder with von-Mises Fisher distributional sliced Wasserstein distance}
\begin{algorithmic}
\label{alg:trainingvDSW}
\STATE \textbf{Input:} Point-cloud distribution $p(X)$,  learning rate $\eta$, slice learning rate $\eta_s$, model maximum number of iterations $\mathcal{T}$,  slice maximum number of iterations $T$, mini-batch size $k$, the number of projections $L$, and the concentration hyperparameter $\kappa$.
  \STATE \textbf{Initialization:} Initialize the encoder $f_\phi$ and the decoder $g_\gamma$
  \WHILE{$\phi,\gamma$ not converge or reach $\mathcal{T}$}
  \STATE Sample a mini-batch $X_1,\ldots,X_k$ i.i.d from $p(X)$
  \STATE $\nabla_\phi = 0,\nabla_\gamma =0$
  \FOR{$i=1$ to $k$}
  \STATE Initialize $\epsilon$
   \WHILE{$\epsilon$ not converge or reach $T$}
\STATE Sample $\theta_1^\epsilon,\ldots,\theta_L^\epsilon $ i.i.d from $\text{vMF}(\epsilon,\kappa)$ via the reparameterized acceptance-rejection sampling in Algorithm~\ref{Alg:vMF_sampling}
  \STATE $\epsilon = \epsilon +  \eta_s \cdot \frac{1}{L}\sum_{l=1}^L\nabla_\epsilon\text{W}_p (\theta_l^\epsilon \sharp P_{X_i},\theta_l^\epsilon \sharp P_{g_\gamma(f_\phi(X_i))})$ \# Other update rules can be used
  \STATE $\epsilon = \frac{\epsilon}{||\epsilon||_2}$ \#Project back to the unit-hypersphere $\mathbb{S}^{d-1}$
  \ENDWHILE
  \STATE Sample $\theta_1^\epsilon,\ldots,\theta_L^\epsilon $ i.i.d from $\text{vMF}(\epsilon,\kappa)$ via the reparameterized acceptance-rejection sampling in Algorithm~\ref{Alg:vMF_sampling}.
  \STATE $\nabla_\phi = \nabla_\phi+\frac{1}{k}  \frac{1}{L} \sum_{i=l}^L \nabla_\phi \text{W}_p (\theta_l^\epsilon \sharp P_{X_i},\theta_l^\epsilon \sharp P_{g_\gamma(f_\phi(X_i))})$
  \STATE $\nabla_\gamma = \nabla_\gamma+\frac{1}{k}\frac{1}{L} \sum_{i=l}^L  \nabla_\gamma \text{W}_p (\theta_l^\epsilon \sharp P_{X_i},\theta_l^\epsilon \sharp P_{g_\gamma(f_\phi(X_i))})$
  \ENDFOR
  \STATE $\phi = \phi - \eta \cdot \nabla_\phi$  \# Other update rules can be used
  \STATE $\gamma = \gamma - \eta \cdot \nabla_\gamma$  \# Other update rules can be used
  \ENDWHILE
 \STATE \textbf{Return:} $\phi,\gamma$
\end{algorithmic}
\end{algorithm}

\textbf{Training point-cloud autoencoder with amortized distributonal projection optimization:} We present the training algorithm for  point-cloud autoencoder with amortized distributional projection optimization in Algorithm~\ref{alg:trainingamortizedvDSW}. With amortized 
 distributional optimization, the inner loop for finding the best distribution over projecting directions is removed.

\begin{algorithm}[!t]
\caption{Training point-cloud autoencoder  with amortized projection optimization}
\begin{algorithmic}
\label{alg:trainingamortizedvDSW}
\STATE \textbf{Input:} Point-cloud distribution $p(X)$,  learning rate $\eta$, slice learning rate $\eta_s$, model maximum number of iterations $\mathcal{T}$,  mini-batch size $k$.
  \STATE \textbf{Initialization:} Initialize the encoder $f_\phi$, the decoder $g_\gamma$, and the amortized model $a_\psi$
  \WHILE{$\phi,\gamma,\psi$ not converge or reach $\mathcal{T}$}
  \STATE Sample a mini-batch $X_1,\ldots,X_k$ i.i.d from $p(X)$
  \STATE $\nabla_\phi = 0,\nabla_\gamma =0,\nabla_\psi = 0$
  \FOR{$i=1$ to $k$} 
  \STATE $\epsilon_{\psi,\gamma,\phi} = a_\psi (X_i,g_\gamma(f_\phi(X_i)))$ 
  \STATE Sample $\theta_1^{\psi,\gamma,\phi},\ldots,\theta_L^{\psi,\gamma,\phi} $ i.i.d from $\text{vMF}(\epsilon_{\psi,\gamma,\phi},\kappa)$ via the reparameterized acceptance-rejection sampling in Algorithm~\ref{Alg:vMF_sampling}
   \STATE $\nabla_\psi = \nabla_\psi+\frac{1}{k} \frac{1}{L} \sum_{i=l}^L  \nabla_\psi \text{W}_p (\theta_l^{\psi,\gamma,\phi} \sharp P_{X_i},\theta_l^{\psi,\gamma,\phi} \sharp P_{g_\gamma(f_\phi(X_i))})$
  \STATE $\nabla_\phi = \nabla_\phi+\frac{1}{k}\frac{1}{L} \sum_{i=l}^L  \nabla_\phi \text{W}_p (\theta_l^{\psi,\gamma,\phi} \sharp P_{X_i},\theta_l^{\psi,\gamma,\phi} \sharp P_{g_\gamma(f_\phi(X_i))})$
  \STATE $\nabla_\gamma = \nabla_\gamma+\frac{1}{k} \frac{1}{L} \sum_{i=l}^L \nabla_\gamma \text{W}_p (\theta_l^{\psi,\gamma,\phi} \sharp P_{X_i},\theta_l^{\psi,\gamma,\phi} \sharp P_{g_\gamma(f_\phi(X_i))})$
  \ENDFOR
  \STATE $\psi = \psi + \eta_s \cdot \nabla_\psi$  \# Other update rules can be used
  \STATE $\phi = \phi - \eta \cdot \nabla_\phi$  \# Other update rules can be used
  \STATE $\gamma = \gamma - \eta \cdot \nabla_\gamma$  \# Other update rules can be used
  \ENDWHILE
 \STATE \textbf{Return:} $\phi,\gamma$
\end{algorithmic}
\end{algorithm}

\section{Proofs}
\label{sec:proofs}

\subsection{Proof for Proposition~\ref{prop:pseudo_metricity}}
\label{subsec:proof:pseudo_metricity}
We first recall the definition of the projected one-dimensional Wasserstein between two probability measures $\mu$ and $\nu$: $\text{PW}_p(\mu,\nu;\hat{\theta}) = \text{W}_p(\hat{\theta}\sharp \mu,\hat{\theta}\sharp \nu)$ for $\hat{\theta} \neq \text{argmax}_{\theta \in \mathbb{S}^{d-1}} \text{W}_p(\theta \sharp \mu,\theta \sharp \nu)$.

\textbf{Non-negativity and Symmetry:} Due to the non-negativity and symmetry of the Wasserstein distance, the non-negativity and symmetry of the projected Wasserstein follow directly from its definition.

\textbf{Triangle inequality: } For any three probability measures $\mu_1,\mu_2,\mu_3 \in \mathcal{P}_p(\mathbb{R}^d)$, we have:
\begin{align*}
    \text{PW}_p(\mu_1,\mu_3;\hat{\theta}) &=  \text{W}_p (\hat{\theta}\sharp \mu_1,\hat{\theta}\sharp \mu_3 ) \\
    &\leq  \text{W}_p (\hat{\theta}\sharp \mu_1,\hat{\theta}\sharp \mu_2 )+ \text{W}_p (\hat{\theta}\sharp \mu_2,\hat{\theta}\sharp \mu_3 )\\
    &=  \text{PW}_p(\mu_1,\mu_2;\hat{\theta}) + \text{PW}_p(\mu_2,\mu_3;\hat{\theta}),
\end{align*}
where the first inequality is due to the triangle inequality of the Wasserstein distance.

\textbf{Identity: } If $\mu =\nu$, we have $\text{PW}_p(\mu,\nu;\hat{\theta}) =0$ due to the identity of the Wasserstein distance. However, if $\text{PW}_p(\mu,\nu;\hat{\theta}) =0$, there exists $\theta' \in \mathbb{S}^{d-1}$ such that $0=\text{PW}_p(\mu,\nu;\hat{\theta}) < \text{PW}_p(\mu,\nu;\theta') $. Let $\mathcal{F}[\gamma](w) = \int_{\mathbb{R}^{d'}} e^{-i \langle w,x\rangle} d \gamma(x)$ be the Fourier transform of $\gamma \in \mathcal{P}(\mathbb{R}^{d'})$, for any $t\in \mathbb{R}$, we have 
\begin{align*}
       \mathcal{F}[\mu](t\theta') &= \int_{\mathbb{R}^{d}} e^{-it\langle \theta',x \rangle} d\mu(x)=  \int_{\mathbb{R}}e^{-itz} d \theta' \sharp \mu(z) = \mathcal{F}[\theta' \sharp \mu](t) \\
       &\neq \mathcal{F}[\theta' \sharp \nu](t) =\int_{\mathbb{R}}e^{-itz} d \theta' \sharp \nu(z) =\int_{\mathbb{R}^{d}} e^{-it\langle \theta',x \rangle} d\nu(x)=\mathcal{F}[\nu](t\theta').
    \end{align*}
Therefore, we have $\mu \neq \nu$. We complete the proof.
\subsection{Proof for Theorem~\ref{theo:reconstruction}}
\label{subsec:proof:reconstruction}

We first start with proving the metricity of the \textit{non-optimal} von Mises Fisher distributional sliced Wasserstein distance (v-DSW). For any two probability measures $\mu,\nu \in \mathcal{P}_p(\mathbb{R}^{d})$, the \textit{non-optimal} von Mises Fisher distributional sliced Wasserstein distance (v-DSW) is defined as follow:
\begin{align*}
    \text{v-DSW}_p (\mu,\nu;\epsilon,\kappa)= \left(\mathbb{E}_{\theta \sim \text{vMF}(\epsilon,\kappa)} \text{W}_p^p (\theta \sharp \mu,\theta \sharp \nu)\right)^{\frac{1}{p}},
\end{align*}
where $\epsilon \in \mathbb{S}^{d-1}$ and  $0 <\kappa <\infty$.
\begin{lemma}
\label{lemma:vDSW}
    For any $\epsilon \in \mathbb{S}^{d-1}$ and $\kappa < \infty$, $\text{v-DSW}_p(\cdot,\cdot;\epsilon,\kappa)$ is a valid metric on the space of probability measures.
\end{lemma}
\begin{proof}
We now prove that v-DSW satisfies non-negativity, symmetry, triangle inequality, and identity.

\textbf{Non-negativity and Symmetry:}  The non-negativity and symmetry of v-DSW follow directly the non-negativity and symmetry of the Wasserstein distance.

\textbf{Triangle inequality:} For any three probability measures $\mu_1,\mu_2,\mu_3 \in \mathcal{P}_p(\mathbb{R}^d)$, we have
\begin{align*}
     \text{v-DSW}_p (\mu_1,\mu_3;\epsilon,\kappa) &=  \left(\mathbb{E}_{\theta \sim \text{vMF}(\epsilon,\kappa)} \text{W}_p^p (\theta \sharp \mu_1,\theta \sharp \mu_3)\right)^{\frac{1}{p}} \\
     &\leq  \left(\mathbb{E}_{\theta \sim \text{vMF}(\epsilon,\kappa)} \left[\text{W}_p (\theta \sharp \mu_1,\theta \sharp \mu_2)+\text{W}_p (\theta \sharp \mu_2,\theta \sharp \mu_3) \right]^p\right)^{\frac{1}{p}} \\
     &\leq \left(\mathbb{E}_{\theta \sim \text{vMF}(\epsilon,\kappa)} \text{W}_p^p (\theta \sharp \mu_1,\theta \sharp \mu_2) \right)^{\frac{1}{p}}+ \left( \mathbb{E}_{\theta \sim \text{vMF}(\epsilon,\kappa)} \text{W}_p^p (\theta \sharp \mu_2,\theta \sharp \mu_3) \right)^{\frac{1}{p}} \\
     &=  \text{v-DSW}_p (\mu_1,\mu_2;\epsilon,\kappa) + \text{v-DSW}_p (\mu_2,\mu_3;\epsilon,\kappa) 
\end{align*}

\textbf{Identity:} From the definition, if $\mu=\nu$, we obtain $\text{v-DSW}_p (\mu,\nu;\epsilon,\kappa) = 0$. Now, we need to show that if $\text{v-DSW}_p 
(\mu,\nu;\epsilon,\kappa) = 0$, then $\mu=\nu$.

If $\text{v-DSW}_p (\mu,\nu;\epsilon,\kappa) = 0$, we have $\left(\mathbb{E}_{\theta \sim \text{vMF}(\epsilon,\kappa)} \text{W}_p^p (\theta \sharp \mu,\theta \sharp \nu)\right)^{\frac{1}{p}}=0$ which implies $\mathbb{E}_{\theta \sim \text{vMF}(\epsilon,\kappa)} \text{W}_p^p (\theta \sharp \mu,\theta \sharp \nu)=0$. Therefore, $\text{W}_p (\theta \sharp \mu,\theta \sharp \nu)=0$ for $\text{vMF}(\epsilon,\kappa)$ almost surely $\theta \in \mathbb{S}^{d-1}$. Using the identity property of the Wasserstein distance, we obtain $\theta \sharp \mu = \theta \sharp \nu$ for $\text{vMF}(\epsilon,\kappa)$ almost surely $\theta \in \mathbb{S}^{d-1}$. Since $\text{vMF}(\epsilon,\kappa)$ with $0<\kappa<\infty$ has the supports on all $\mathbb{S}^{d-1}$, for any $t \in \mathbb{R}$ and $\theta \in \mathbb{S}^{d-1}$, we have:
\begin{align*}
       \mathcal{F}[\mu](t\theta) &= \int_{\mathbb{R}^{d}} e^{-it\langle \theta,x \rangle} d\mu(x)=  \int_{\mathbb{R}}e^{-itz} d \theta \sharp \mu(z) = \mathcal{F}[\theta \sharp \mu](t) \\
       &=\mathcal{F}[\theta \sharp \nu](t) =\int_{\mathbb{R}}e^{-itz} d \theta \sharp \nu(z) =\int_{\mathbb{R}^{d}} e^{-it\langle \theta,x \rangle} d\nu(x)=\mathcal{F}[\nu](t\theta), 
    \end{align*}
where $\mathcal{F}[\gamma](w) = \int_{\mathbb{R}^{d'}} e^{-i \langle w,x\rangle} d \gamma(x)$ denotes the Fourier transform of $\gamma \in \mathcal{P}(\mathbb{R}^{d'})$.  We then obtain $\mu = \nu$ by the injectivity of the Fourier transform.  We complete the proof.
\end{proof}

By abuse of notation, we denote $\text{v-DSW}(X,Y;\epsilon,\kappa) =  \text{v-DSW}(P_X,P_Y;\epsilon,\kappa)$ for $X,Y \in \mathcal{X}$ are two point-clouds, $P_X =  \frac{1}{m} \sum_{i=1}^m \delta_{x_i}$, and $P_Y =  \frac{1}{m} \sum_{i=1}^m \delta_{y_i}$. We cast the v-DSW from a metric on the space of probability measures to the space of point-clouds $\mathcal{X}$.
\begin{corollary}
\label{coroll:metric}
    For any $\epsilon \in \mathbb{S}^{d-1}$ and $\kappa < \infty$, $\text{v-DSW}_p(\cdot,;\epsilon,\kappa)$ is a valid metric on the space of point-clouds $\mathcal{X}$.
\end{corollary}
\begin{proof}
    Since $P_X,P_Y \in \mathcal{P}_p(\mathbb{R}^d)$, the non-negativity, symmetry, triangle inequality, and identity properties follow directly from Lemma~\ref{lemma:vDSW}. We now only need to show that v-DSW is invariant to permutation. This property is straightforward from the definition of empirical probability measures. For any permutation function $\sigma$, we have $P_X  =  \frac{1}{m} \sum_{i=1}^m \delta_{x_i} =  \frac{1}{m} \sum_{i=1}^m \delta_{x_\sigma(i)} =P_{\sigma(X)}$ which completes the proof.
\end{proof}

We now continue the proof of Theorem~\ref{theo:reconstruction}. If $\mathbb{E}_{X \sim p(X)} \left(\mathbb{E}_{\theta \sim \text{vMF}(\epsilon,\kappa)}  \text{W}_p^p(\theta \sharp P_X,\theta \sharp P_{g_\gamma (f_\phi(X))}) \right)^{\frac{1}{p}}= 0$, we obtain $\left(\mathbb{E}_{\theta \sim \text{vMF}(\epsilon,\kappa)}  \text{W}_p^p(\theta \sharp P_X,\theta \sharp P_{g_\gamma (f_\phi(X))}) \right)^{\frac{1}{p}}=\text{v-DSW}(X,g_\gamma(f_\phi(X));\epsilon,\kappa)= 0$ for p-almost surely $X \in \mathcal{X}$. By Collorary~\ref{coroll:metric}, we obtain $X=g_\gamma(f_\phi(X))$  for p-almost surely $X \in \mathcal{X}$. We complete the proof.

\subsection{Proof for Proposition~\ref{prop:invariance}}
\label{subsec:proof:invariance}

We first recall the definition of the self-attention amortized model in Definition~\ref{def:efficient_attention_model}: 

\begin{align*}
    a_\psi (X,Y)=\frac{\gA_{\zeta}(X'^\top)^\top \vone_{m} +  \gA_{\zeta}(Y'^\top)^\top \vone_{m}}{||\gA_{\zeta}(X'^\top)^\top  \vone_{m} +  \gA_{\zeta}(Y'^\top)^\top \vone_{m}||_2},
\end{align*}

\textbf{Symmetry:} Since the self-attention amortized model use the same attention weight $\zeta$ for both $X$ and $Y$, exchanging $X$ and $Y$ yields the same results $a_\psi(X,Y) =  a_\psi(Y,X)$.

\textbf{Permutation invariance:} Based on the results in~\citet[Appendix A]{yang2019modeling}, we show that self-attention amortized model is permutation invariant. In particular, we have:
\begin{align*}
    \gA_{\zeta}(X'^\top)^\top \vone_{m}  &= \attention(X'^\top W_q, X'^\top W_k, X'^\top  W_v)^\top \vone_{m}  \\
    &= \left(\text{softmax}_{\text{row}} \left[ \frac{X'^\top  W_q W_k^\top X'}{\sqrt{d_k}}\right] X'^\top  W_v\right)^\top \vone_{m}  \\
    &=  \left(\text{softmax}_{\text{row}} \left[ \frac{\sigma(X)'^\top  W_q W_k^\top \sigma(X)'}{\sqrt{d_k}}\right] \sigma(X)'^\top  W_v\right)^\top \vone_{m}   \\
    &= \gA_{\zeta}(\sigma(X)'^\top)^\top \vone_{m} .
\end{align*}
Similarly, the proof holds for both linear self-attention and efficient self-attention.

\section{Experiment settings}
\label{sec:exp_settings}
In this section, we first provide the details of the training process and the architecture for point-cloud reconstruction, transfer learning, and point-cloud generation. Then, we present the implementation detail and hyper-parameters settings for different distances used in our experiments.

\subsection{Details of point-cloud reconstruction and downstream applications}
\label{subsec:ae_details}
\textbf{Point-cloud reconstruction:} We use the same settings in ASW~\cite{Nguyen2021PointSetDistances} to train autoencoders. We utilize a variant of Point-Net~\cite{qi2017pointnet} with an embedding size of 256 proposed in~\cite{pham2020lcd}. The architecture of the autoencoder and classifier are shown in Figure~\ref{fig:architecture}. Our autoencoder is trained on the ShapeNet Core-55 dataset~\cite{shapenet2015} with a batch size of 128 and a point-cloud size of 2048. We train it for 300 epochs using an SGD optimizer with an initial learning rate of 1e-3, a momentum of 0.9, and a weight decay of 5e-4. 

\begin{figure}[t]
\begin{center}
\begin{tabular}{c}
\widgraph{0.6\textwidth}{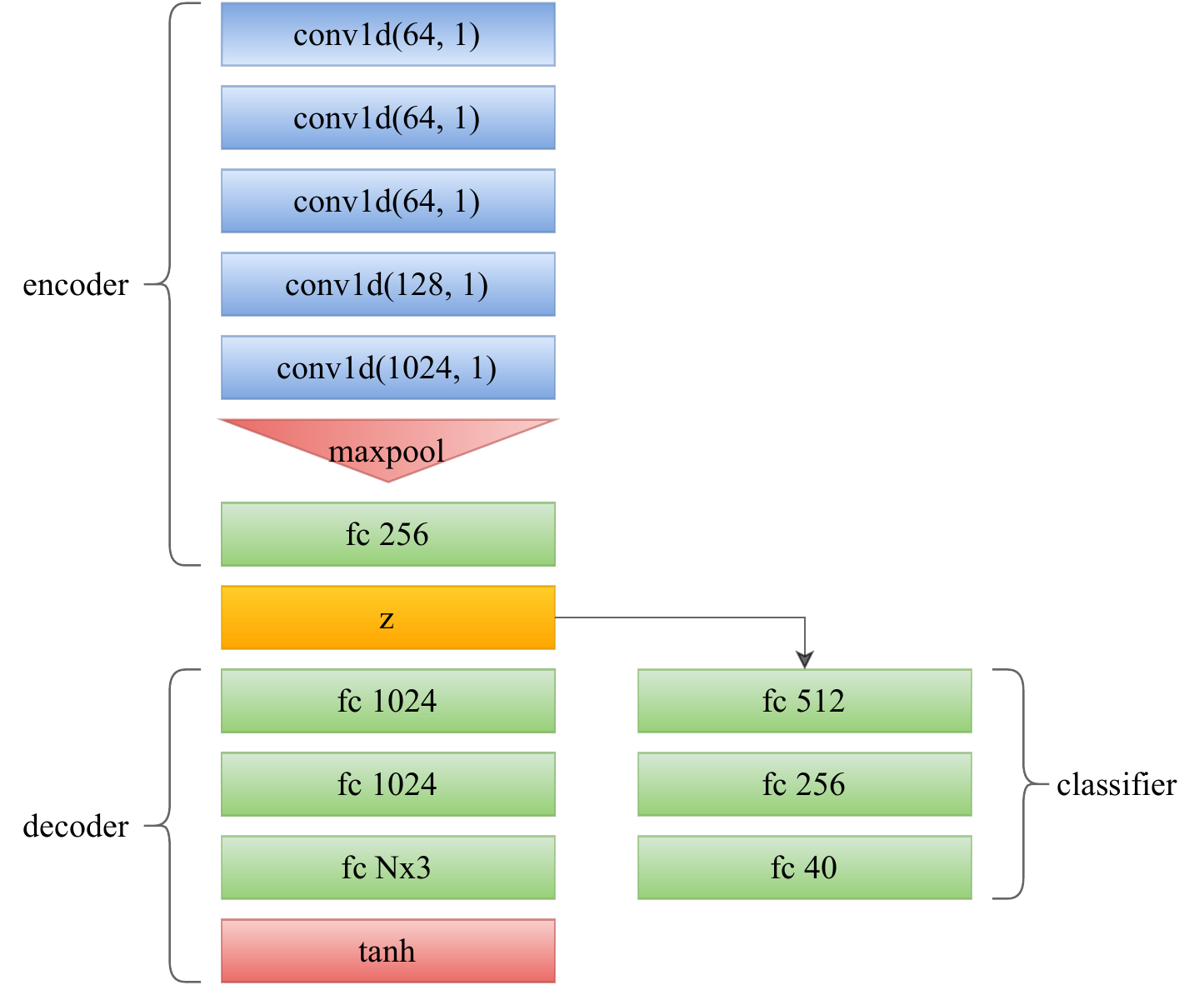}
\end{tabular}
\end{center}
\vskip -0.2in
\caption{
\footnotesize{The architecture of the Point-Net variant in our experiments. For transfer learning, we use a simple classifier with 3 fully-connected layers. All layers are followed by ReLU activation and batch normalization by default, except for the final layers.}} 
\label{fig:architecture}
\vskip -0.2in
\end{figure}

Next, we detail the process of conducting two downstream applications of point-cloud reconstruction.

\textbf{Transfer learning:} A classifier is trained on the latent space of the autoencoder. Particularly, we extract a 256-dimension (which is smaller than the setting in~\cite{lee2022statistical}) latent vector of an input 3D point-cloud via the pre-trained encoder. Then, this vector is fed into a multi-layer perceptron with hidden layers of size 512 and 256. The last layer outputs a 40-dimension vector representing the prediction of 40 classes of the ModelNet40 dataset.

\textbf{Point-cloud generation:} Our generative model is trained on the latent space of the autoencoder as follows. First, we extract a 256-dimension latent vector of an input 3D point-cloud via the pre-trained encoder. Then a 64-dimensional vector is drawn from a normal distribution $\gN(0, \sI_{64})$, where $\sI_{64}$ is the $64 x 64$ identity matrix, and fed into a generator which also outputs a 256-dimension vector. Finally, the generator learns by minimizing the optimal transport distance between the generated and ground truth latent codes.

\subsection{Details of baseline distances}
\label{subsec:baseline_details}
We want to emphasize that we use the same set of hyper-parameters reported in~\cite{Nguyen2021PointSetDistances} for Chamfer, EMD, SW, and Max-SW.

\textbf{Chamfer and EMD:} We use the CUDA implementation from~\cite{yang2019pointflow}.

\textbf{SW:} We use the Monte Carlo estimation with 100 slices.

\textbf{Max-SW:} We use the projected sub-gradient ascent algorithm to optimize the projection. It is trained with an Adam optimizer with an initial learning rate of 1e-4. The number of iterations T is chosen from $\{1, 10, 50\}$.

\textbf{Adaptive SW:} We use Algorithm 1 in~\cite{Nguyen2021PointSetDistances} with the same set of parameters as follows: $N_0 = 2, s = 1, \epsilon = 0.5, $ and $M = 500$.
\begin{table}[t]
    \caption{Reconstruction and transfer learning performance of different autoencoders on the ModelNet40 dataset. For v-DSW and Max-SW, T denotes the number of projected sub-gradient ascent iterations. In Table~\ref{table:short_reconstruction_result}, both v-DSW and Max-SW have T = 50 iterations. All reconstructed losses except EMD are multiplied by 100.}
    \vskip 0.1in
    \label{table:full_reconstruction_result}
    \centering
    \scalebox{0.9}{
        \begin{tabular}{lccccc}
        \toprule
        Method & CD $(10^{-2}, \downarrow)$ & SW $(10^{-2}, \downarrow)$ & EMD $(\downarrow)$ & Acc $(\uparrow)$ & Time $(\downarrow)$ \\
        \midrule
        CD & 1.25 $\pm$ 0.03 & 681.20 $\pm$ 16.73 & 653.52 $\pm$ 10.43 & 86.28 $\pm$ 0.34 & 95 \\
        EMD & 0.40 $\pm$ 0.00 & 94.54 $\pm$ 2.90 & 168.60 $\pm$ 1.57 & 88.45 $\pm$ 0.20 & 208 \\
        \midrule
        SW & 0.68 $\pm$ 0.01 & 89.61 $\pm$ 3.88 & 191.12 $\pm$ 2.88 & 87.90 $\pm$ 0.27 & 106 \\
        Max-SW (T = 1) & 0.69 $\pm$ 0.01 & 87.60 $\pm$ 0.95 & 190.88 $\pm$ 0.40 & 88.05 $\pm$ 0.23 & 97 \\
        Max-SW (T = 10) & 0.69 $\pm$ 0.01 & 90.72 $\pm$ 0.58 & 192.82 $\pm$ 0.73 & 87.82 $\pm$ 0.37 & 102 \\
        Max-SW (T = 50) & 0.68 $\pm$ 0.01 & 88.22 $\pm$ 1.45 & 190.23 $\pm$ 0.1 & 87.97 $\pm$ 0.14 & 116 \\
        ASW & 0.69 $\pm$ 0.01 & 89.42 $\pm$ 5.07 & 192.03 $\pm$ 3.09 & 87.78 $\pm$ 0.20 & 103 \\
        v-DSW (T = 1) & \textbf{0.67 $\pm$ 0.01} & 87.29 $\pm$ 1.49 & 188.52 $\pm$ 1.47 & 87.87 $\pm$ 0.28 & 115 \\
        v-DSW (T = 10) & 0.68 $\pm$ 0.00 & 87.44 $\pm$ 1.07 & 189.97 $\pm$ 1.04 & 87.98 $\pm$ 0.23 & 205 \\
        v-DSW (T = 50) & \textbf{0.67 $\pm$ 0.00} & 85.03 $\pm$ 3.31 & 187.75 $\pm$ 2.00 & 87.83 $\pm$ 0.40 & 633 \\
        $\gL$-Max-SW & 1.06 $\pm$ 0.03 & 121.85 $\pm$ 5.77 & 236.87 $\pm$ 3.42 & 87.70 $\pm$ 0.23 & \textbf{94} \\
        $\gG$-Max-SW & 12.11 $\pm$ 0.29 & 851.07 $\pm$ 2.11 & 829.28 $\pm$ 5.53 & 87.49 $\pm$ 0.36 & 97 \\
        $\gN$-Max-SW & 7.38 $\pm$ 3.29 & 618.74 $\pm$ 153.87 & 648.32 $\pm$ 117.03 & 87.43 $\pm$ 0.15 & 96 \\
        \midrule
        $\gL$v-DSW (ours) & 0.68 $\pm$ 0.00 & 85.32 $\pm$ 0.54 & 188.32 $\pm$ 0.23 & 87.70 $\pm$ 0.34 & 114 \\
        $\gG$v-DSW (ours) & 0.68 $\pm$ 0.01 & 82.77 $\pm$ 0.48 & 187.04 $\pm$ 1.11 & 87.75 $\pm$ 0.19 & 117 \\
        $\gN$v-DSW (ours) & \textbf{0.67 $\pm$ 0.00} & 83.47 $\pm$ 0.49 & 186.66 $\pm$ 0.81 & 87.84 $\pm$ 0.07 & 115 \\
        $\gA$v-DSW (ours) & 0.67 $\pm$ 0.01 & 83.08 $\pm$ 1.22 & 186.27 $\pm$ 0.56 & 88.05 $\pm$ 0.17 & 230 \\
        $\gE\gA$v-DSW (ours) & 0.68 $\pm$ 0.01 & 82.05 $\pm$ 0.40 & 186.46 $\pm$ 0.25 & 88.07 $\pm$ 0.21 & 125 \\
        $\gL\gA$v-DSW (ours) & 0.68 $\pm$ 0.00 & \textbf{81.03 $\pm$ 0.18} & \textbf{185.26 $\pm$ 0.31} & \textbf{88.28 $\pm$ 0.13} & 123 \\
        \bottomrule
        \end{tabular}
    }
    \vskip -0.2in
\end{table}

\textbf{v-DSW:} We use stochastic projected gradient ascent algorithm to optimize the location vector $\epsilon$ in Equation~\ref{def:vMF} while we fix the concentration parameter $\kappa$ to 1 for both v-DSW and all of its amortized versions. Similar to Max-SW, it is trained with an Adam optimizer with an initial learning rate of 1e-4. The number of iterations $T$ is selected from $\{ 1, 10, 50 \}$ based on the task performance. Intuitively, increasing the number of iterations leads to a better approximation that is closer to the optimal value but comes with an expensive computational cost. We also use the Monte Carlo estimation with 100 slices as in SW.

\subsection{Details of amortized sliced Wasserstein distances}
\label{subsec:amortize_details}
\textbf{Linear, generalized linear, and non-linear models:} We adopt the official implementations in~\cite{nguyen2022amortized}.

\textbf{Self-attention-based models:} We adapt the official implementations from their corresponding papers in our experiments. For all variants, $d_v$ is set to 3, which equals the dimension of point-clouds while $d_k$ is chosen from $\{ 16, 32, 64, 128\}$. In Equation~\ref{eq:linear_self_attention}, the projected dimension $k$ is selected from $\{ 64, 128 \}$.

\begin{figure}[!t]
\begin{center}
\begin{tabular}{c}
\widgraph{0.95\textwidth}{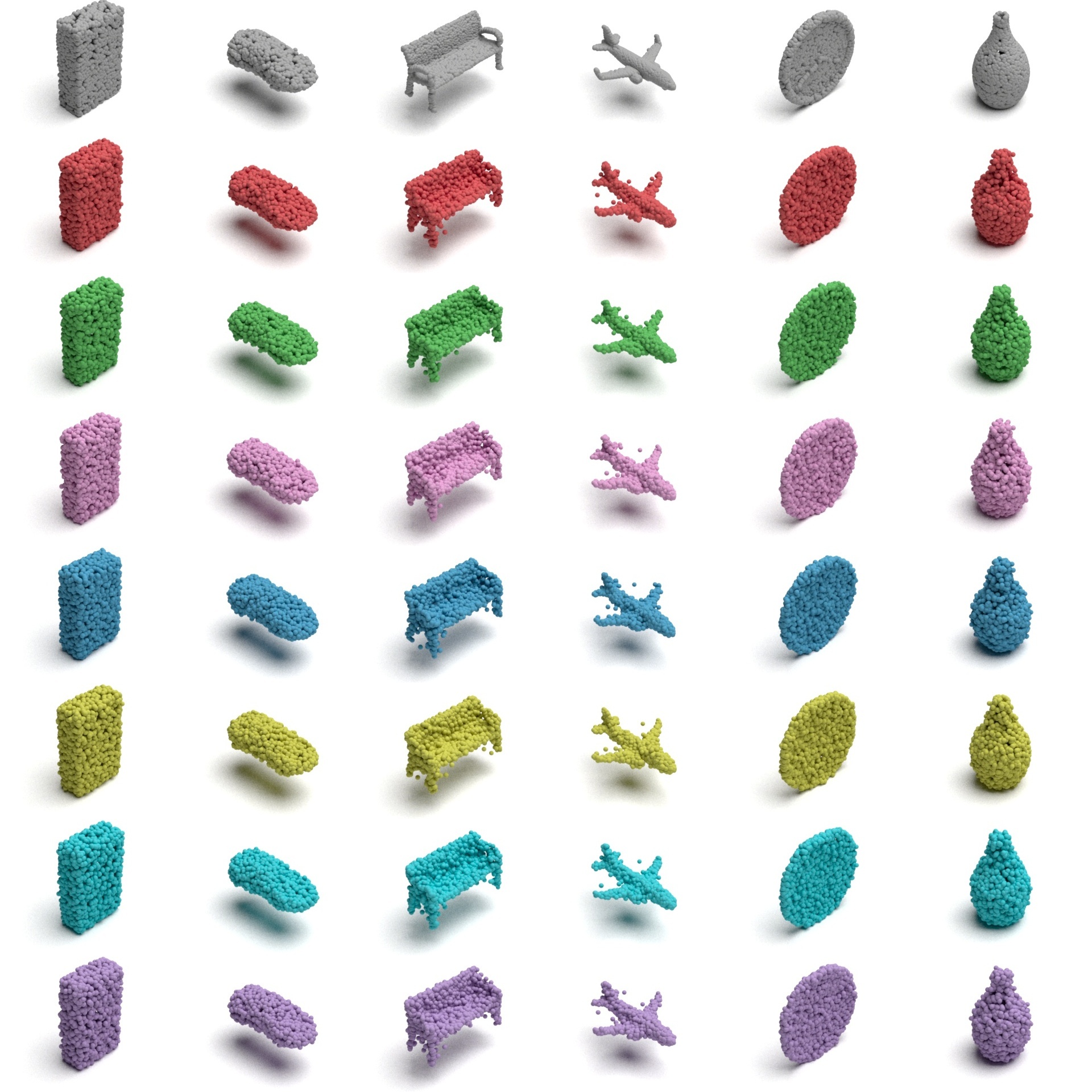}
\end{tabular}
\end{center}
\caption{
\footnotesize{Qualitative results of reconstructing point-clouds in the ShapeNet Core-55 dataset. From top to bottom: input, SW, Max-SW (T = 50), ASW, v-DSW (T = 50), $\gN$v-DSW, $\gE\gA$v-DSW, and $\gL\gA$v-DSW.}
} 
\label{fig:reconstructed_point_clouds_full}
\vskip -0.2in
\end{figure}

\textbf{Training amortized models:} The learning rate is set to 1e-3 and the optimizer is set to Adam~\cite{kingma2014adam} with $(\beta_1, \beta_2) = (0, 0.9)$.

\begin{table}[t!]
    \caption{Quantitative results (measured in EMD) of reconstructing point-clouds in the ShapeNet Core-55 dataset.}
    \vskip 0.1in
    \label{table:full_reconstruction_quantitative}
    \centering
    \scalebox{1.0}{
        \begin{tabular}{cccccccc}
        \toprule
        Method & PC1 & PC2 & PC3 & PC4 & PC5 & PC6 & Avg \\
        \midrule
        SW & 141.07 & 139.50 & 118.83 & 99.11 & 150.28 & 128.46 & 129.54 \\
        Max-SW (T = 50) & 145.15 & 131.76 & 112.13 & 116.73 & 139.91 & 115.79 & 126.91 \\
        ASW & 139.17 & 126.55 & 115.49 & \textbf{91.07} & 153.87 & 114.84 & 123.50 \\
        v-DSW (T = 50) & 133.06	& 146.99 & 105.65 & 105.66 & 137.32 & \textbf{110.50} & 123.20 \\
        $\gN$v-DSW & 132.60 & 127.57 & 100.81 & 94.31 & 131.04 & 116.34 & 117.11 \\
        $\gE\gA$v-DSW & 139.64 & \textbf{124.28} & 100.34 & 98.33 & \textbf{123.59} & 115.05 & 116.87 \\
        $\gL\gA$v-DSW & \textbf{130.21} & 127.00 & \textbf{96.75} & 98.09 & 132.33 & 114.11 & \textbf{116.41} \\
        \bottomrule
        \end{tabular}
    }
    \vskip -0.1in
\end{table}

\begin{table}[t!]
    \caption{Reconstruction results of SW and $\gL\gA$v-DSW when changing $L$. CD and SWD are multiplied by 100.}
    \vskip 0.1in
    \label{table:ablation_L}
    \centering
    \scalebox{1.0}{
        \begin{tabular}{cccccc}
        \toprule
        Method & $L$ & CD $(10^{-2}, \downarrow)$ & SW $(10^{-2}, \downarrow)$ & EMD $(\downarrow)$ & Time \\
        \midrule
        \multirow{4}{*}{SW} & 50 & \textbf{0.67 $\pm$ 0.00} & 90.17 $\pm$ 2.97 & 190.97 $\pm$ 1.87 & 100 \\
        & 100 & 0.68 $\pm$ 0.01 & \textbf{89.61 $\pm$ 3.88} & 191.12 $\pm$ 2.88 & 107 \\
        & 200 & \textbf{0.67 $\pm$ 0.00} & 89.54 $\pm$ 4.57 & 191.21 $\pm$ 3.87 & 111 \\
        & 500 & 0.67 $\pm$ 0.01 & 88.20 $\pm$ 4.22 & 190.14 $\pm$ 2.35 & 142 \\
        \midrule
        \multirow{2}{*}{$\gL\gA$v-DSW} & 50 & 0.68 $\pm$ 0.01 & 85.88 $\pm$ 4.03 & 188.80 $\pm$ 2.55 & 133 \\
        & 100 & 0.68 $\pm$ 0.00 & \textbf{81.03 $\pm$ 0.18} & \textbf{185.26 $\pm$ 0.31} & 123 \\
        \bottomrule
        \end{tabular}
    }
    \vskip -0.1in
\end{table}

\begin{table}[t!]
    \caption{Reconstruction results of v-DSW when changing the number of projected sub-gradient ascent iteration (T). CD and SWD are multiplied by 100.}
    \vskip 0.1in
    \label{table:ablation_T}
    \centering
    \scalebox{1.0}{
        \begin{tabular}{cccc}
        \toprule
        Method & CD $(10^{-2}, \downarrow)$ & SW $(10^{-2}, \downarrow)$ & EMD $(\downarrow)$ \\
        \midrule
        v-DSW (T = 0) & 0.67 $\pm$ 0.01 & 88.63 $\pm$ 2.30 & 189.81 $\pm$ 1.19 \\
        v-DSW (T = 1) & 0.67 $\pm$ 0.01 & 87.29 $\pm$ 1.49 & 188.52 $\pm$ 1.47 \\
        v-DSW (T = 10) & 0.68 $\pm$ 0.00 & 87.44 $\pm$ 1.07 & 189.97 $\pm$ 1.04 \\
        v-DSW (T = 50) & \textbf{0.67 $\pm$ 0.00} & 85.03 $\pm$ 3.31 & 187.75 $\pm$ 2.00 \\
        \midrule
        $\gL\gA$v-DSW & 0.68 $\pm$ 0.00 & \textbf{81.03 $\pm$ 0.18} & \textbf{185.26 $\pm$ 0.31} \\
        \bottomrule
        \end{tabular}
    }
    \vskip -0.1in
\end{table}

\section{Additional experimental results}
\label{sec:add_exps}
\textbf{Point-cloud reconstruction:}
Table~\ref{table:full_reconstruction_result} illustrates the full quantitative results of the point-cloud reconstruction experiment. For Max-SW and v-DSW, we vary the number of gradient iterations T in $\{ 1, 10, 50 \}$. Because CD is not a proper distance so we choose the best number of iterations based on SW and EMD losses (we prioritize EMD loss first then SW). As can be seen from the table, increasing the number of gradient ascent iterations $(T)$ increases the reconstruction performance of Max-SW and v-DSW but comes with the cost of additional computation, especially for v-DSW. However, with all choices of T, the reconstruction performance (measured in SW and EMD) of both Max-SW and v-DSW are generally worse than our amortized methods. In addition, our amortized methods have smaller standard deviations over 3 runs, thus they are more stable than the conventional optimization method using gradient ascent method. The qualitative results are given in Figure~\ref{fig:reconstructed_point_clouds_full}. The corresponding quantitative results in EMD are given in Table~\ref{table:full_reconstruction_quantitative}. It can be seen that our amortized v-DSW methods have more favorable performance.

\begin{table}[t]
    \caption{Reconstruction results of $\gL\gA$v-DSW when changing $\kappa$. CD and SWD are multiplied by 100.}
    \vskip 0.1in
    \label{table:ablation_kappa}
    \centering
    \scalebox{1.0}{
        \begin{tabular}{cccc}
        \toprule
        $\kappa$ & CD $(10^{-2}, \downarrow)$ & SW $(10^{-2}, \downarrow)$ & EMD $(\downarrow)$ \\
        \midrule
        0.1 & \textbf{0.67 $\pm$ 0.00} & 81.88 $\pm$ 1.09 & 185.30 $\pm$ 0.94 \\
        1 & 0.68 $\pm$ 0.00 & \textbf{81.03 $\pm$ 0.18} & \textbf{185.26 $\pm$ 0.31} \\
        10 & 0.85 $\pm$ 0.01 & 96.01 $\pm$ 4.24 & 208.46 $\pm$ 4.03 \\
        \bottomrule
        \end{tabular}
    }
    \vskip -0.1in
\end{table}
\begin{table}[!t]
    \caption{Performance comparison of point-cloud generation on the chair category of ShapeNet. For v-DSW and Max-SW, T denotes the number of projected sub-gradient ascent iterations. In Table~\ref{table:short_generation_result}, v-DSW and Max-SW have T = 50 and 10 iterations, respectively. JSD, MMD-CD, and MMD-EMD are multiplied by 100.}
    \vskip 0.1in
    \label{table:full_generation_result}
    \centering
    \scalebox{0.88}{
        \begin{tabular}{lccccccc}
        \toprule
        Method & JSD $(10^{-2}, \downarrow)$ & \multicolumn{2}{c}{MMD $(10^{-2}, \downarrow)$} & \multicolumn{2}{c}{COV $(\%, \uparrow)$} & \multicolumn{2}{c}{1-NNA $(\%, \downarrow)$} \\
        \cmidrule(lr){3-4} \cmidrule(lr){5-6} \cmidrule(lr){7-8}
        & & CD & EMD & CD & EMD & CD & EMD \\
        \midrule
        CD & 17.88 $\pm$ 1.14 & 1.12 $\pm$ 0.02 & 17.19 $\pm$ 0.36 & 23.73 $\pm$ 1.69 & 10.83 $\pm$ 0.89 & 98.45 $\pm$ 0.10 & 100.00 $\pm$ 0.00 \\
        EMD & 5.15 $\pm$ 1.52 & 0.61 $\pm$ 0.09 & 10.37 $\pm$ 0.61 & 41.65 $\pm$ 2.19 & 42.54 $\pm$ 2.42 & 87.76 $\pm$ 1.46 & 87.30 $\pm$ 1.22 \\
        \midrule
        SW & 1.56 $\pm$ 0.06 & 0.72 $\pm$ 0.02 & 10.80 $\pm$ 0.11 & 38.55 $\pm$ 0.43 & 45.35 $\pm$ 0.48 & 89.91 $\pm$ 1.17 & 88.28 $\pm$ 0.70 \\
        Max-SW (T = 1) & 1.74 $\pm$ 0.22 & 0.78 $\pm$ 0.05 & 11.05 $\pm$ 0.31 & 39.39 $\pm$ 2.28 & 46.82 $\pm$ 0.79 & 92.15 $\pm$ 0.95 & 90.20 $\pm$ 0.87 \\
        Max-SW (T = 10) & 1.63 $\pm$ 0.32 & 0.74 $\pm$ 0.01 & 10.84 $\pm$ 0.08 & 40.47 $\pm$ 1.04 & 47.81 $\pm$ 0.78 & 91.46 $\pm$ 0.72 & 89.93 $\pm$ 0.86 \\
        Max-SW (T = 50) & 1.57 $\pm$ 0.26 & 0.80 $\pm$ 0.05 & 11.25 $\pm$ 0.34 & 37.81 $\pm$ 1.69 & 46.23 $\pm$ 0.64 & 92.15 $\pm$ 0.72 & 90.35 $\pm$ 0.28 \\
        ASW & 1.75 $\pm$ 0.38 & 0.78 $\pm$ 0.05 & 11.27 $\pm$ 0.38 & 38.16 $\pm$ 2.15 & 45.45 $\pm$ 1.40 & 91.21 $\pm$ 0.40 & 89.36 $\pm$ 0.40 \\
        v-DSW (T = 1) & 1.84 $\pm$ 0.17 & 0.75 $\pm$ 0.03 & 11.02 $\pm$ 0.21 & 38.26 $\pm$ 1.46 & 45.35 $\pm$ 1.70 & 90.08 $\pm$ 0.48 & 87.81 $\pm$ 0.16 \\
        v-DSW (T = 10) & 1.48 $\pm$ 0.17 & 0.77 $\pm$ 0.02 & 11.09 $\pm$ 0.09 & 37.22 $\pm$ 0.96 & 43.77 $\pm$ 0.39 & 90.40 $\pm$ 1.05 & 88.87 $\pm$ 1.04 \\
        v-DSW (T = 50) & 1.79 $\pm$ 0.17 & 0.72 $\pm$ 0.02 & 10.73 $\pm$ 0.20 & 37.76 $\pm$ 0.71 & 45.49 $\pm$ 1.37 & 90.23 $\pm$ 0.13 & 88.33 $\pm$ 0.95 \\
        \midrule
        $\gL$v-DSW (ours) & 1.67 $\pm$ 0.07 & 0.77 $\pm$ 0.04 & 11.10 $\pm$ 0.33 & 37.91 $\pm$ 1.84 & 45.64 $\pm$ 2.30 & 90.42 $\pm$ 0.53 & 88.82 $\pm$ 0.38 \\
        $\gG$v-DSW (ours) & 1.56 $\pm$ 0.22 & 0.75 $\pm$ 0.02 & 10.99 $\pm$ 0.11 & 37.81 $\pm$ 1.70 & 45.69 $\pm$ 0.46 & 90.32 $\pm$ 0.38 & 88.26 $\pm$ 0.28 \\
        $\gN$v-DSW (ours) & \textbf{1.44 $\pm$ 0.06} & 0.75 $\pm$ 0.02 & 10.95 $\pm$ 0.09 & 38.40 $\pm$ 1.34 & 46.28 $\pm$ 2.06 & 90.15 $\pm$ 0.80 & 88.65 $\pm$ 0.82 \\
        $\gE\gA$v-DSW (ours) & 1.73 $\pm$ 0.21 & \textbf{0.71 $\pm$ 0.04} & \textbf{10.70 $\pm$ 0.26} & 40.03 $\pm$ 1.28 & \textbf{48.01 $\pm$ 1.07} & 89.98 $\pm$ 0.57 & 88.55 $\pm$ 0.38 \\
        $\gL\gA$v-DSW (ours) & 1.54 $\pm$ 0.09 & 0.72 $\pm$ 0.03 & 10.74 $\pm$ 0.35 & \textbf{40.62 $\pm$ 1.39} & 45.84 $\pm$ 1.23 & \textbf{89.44 $\pm$ 0.28} & \textbf{87.79 $\pm$ 0.37} \\
        \bottomrule
        \end{tabular}
    }
\end{table}
\textbf{On the number of projections ($L$).} In our experiments, $L$ is fixed to 100 as in the ASW's paper. Here, we conduct an ablation study on the number of projections $L$ and report the result in Table~\ref{table:ablation_L}. As can be seen from the table, increasing the number of projections improves the performance in terms of EMD but comes with an extra running time. We see that $\gL\gA$v-DSW with $L = 50$ and $L = 100$ are faster than SW with $L = 500$ while being better in terms of SW and EMD evaluation metrics. Compared to SW with $L = 200$, $\gL\gA$v-DSW with $L = 50$ has approximately the same computational time while having lower SW and EMD evaluation metrics. 

\textbf{On the choice of location vector $\epsilon$.} We would like to recall that the optimal location vector $\epsilon^\star$ of v-DSW are computed using Algorithm~\ref{alg:trainingvDSW} in Appendix~\ref{subsec:training_algorithms}. To show its effectiveness, we compare it with a random location $\epsilon$, i.e. T = 0. Table~\ref{table:ablation_T} illustrates that optimizing for the location parameter of the vMF distribution helps to improve the reconstruction. Moreover, our amortized optimization still gives better reconstruction scores than the randomly initialized location $\epsilon$ and the optimized location using the conventional method. Therefore, using amortized optimization could actually have benefits.

\textbf{On the choice of parameter $\kappa$.} We would like first to recall that $\kappa$ is set to 1 for all v-DSW and amortized v-DSW methods in our experiments. In practice, the parameter $\kappa$ can be chosen by doing a grid search. Here, we conduct an ablation study by varying $\kappa \in \{ 0.1, 1, 10 \}$ for $\gL\gA$v-DSW and report the results in Table~\ref{table:ablation_kappa}. As can be seen from the table, $\kappa = 1$ results in the best-performing EMD.

\textbf{Point-cloud generation.}
We summarize the full quantitative results for point-cloud generation in Table~\ref{table:full_generation_result}. For Max-SW and v-DSW, we again change the number of gradient iterations $T$ in $\{ 1, 10, 50 \}$. Note that $\gA$v-DSW cannot be used in this experiment due to being out of memory while the performance of amortized Max-SW is too bad. Therefore, their results are not reported in this experiment. As can be seen from the table, amortized v-DSW methods achieve the best performance in 7 out of 7 metrics. Using more than one gradient ascent iteration $(T \in \{ 10, 50 \})$ does improve the generation performance of Max-SW and v-DSW but comes with the cost of additional computation.


\end{document}